\newif\ifarXiv         % declaration
\newif\ifjournal        % declaration

\journalfalse \arXivtrue  
\ifjournal 
\documentclass[final,onefignum,onetabnum]{siamonline250211}

\usepackage[mathlines]{lineno}
\linenumbers

\usepackage{braket,amsfonts}

\usepackage{array}

\usepackage[caption=false]{subfig}
\usepackage{pgfplots}

\newsiamthm{claim}{Claim}
\newsiamremark{remark}{Remark}
\newsiamremark{hypothesis}{Hypothesis}
\crefname{hypothesis}{Hypothesis}{Hypotheses}

\usepackage{algorithmic}

\usepackage{graphicx,epstopdf}

\Crefname{ALC@unique}{Line}{Lines}

\usepackage{amsopn}

\usepackage{xspace}
\usepackage{bold-extra}
\usepackage[most]{tcolorbox}

\newcounter{example}
\colorlet{texcscolor}{blue!50!black}
\colorlet{texemcolor}{red!70!black}
\colorlet{texpreamble}{red!70!black}
\colorlet{codebackground}{black!25!white!25}

\lstdefinestyle{siamlatex}{%
  style=tcblatex,
  texcsstyle=*\color{texcscolor},
  texcsstyle=[2]\color{texemcolor},
  keywordstyle=[2]\color{texemcolor},
  moretexcs={cref,Cref,maketitle,mathcal,text,headers,email,url},
}

\tcbset{%
  colframe=black!75!white!75,
  coltitle=white,
  colback=codebackground, % bottom/left side
  colbacklower=white, % top/right side
  fonttitle=\bfseries,
  arc=0pt,outer arc=0pt,
  top=1pt,bottom=1pt,left=1mm,right=1mm,middle=1mm,boxsep=1mm,
  leftrule=0.3mm,rightrule=0.3mm,toprule=0.3mm,bottomrule=0.3mm,
  listing options={style=siamlatex}
}

\newtcblisting[use counter=example]{example}[2][]{%
  title={Example~\thetcbcounter: #2},#1}

\newtcbinputlisting[use counter=example]{\examplefile}[3][]{%
  title={Example~\thetcbcounter: #2},listing file={#3},#1}

\DeclareTotalTCBox{\code}{ v O{} }
{ %fontupper=\ttfamily\color{texemcolor},
  fontupper=\ttfamily\color{black},
  nobeforeafter,
  tcbox raise base,
  colback=codebackground,colframe=white,
  top=0pt,bottom=0pt,left=0mm,right=0mm,
  leftrule=0pt,rightrule=0pt,toprule=0mm,bottomrule=0mm,
  boxsep=0.5mm,
  #2}{#1}

\patchcmd\newpage{\vfil}{}{}{}
\usepackage{amssymb}
\usepackage{graphicx}          % graphicx package enables you to paste in graphics
\usepackage{amsmath}         % best maths package on offer
\usepackage{mathtools}
\usepackage{mathrsfs}
\usepackage{bbm}
\usepackage{caption}
\usepackage{dsfont}
\usepackage{listings}
\usepackage{titlepic}
\usepackage{xcolor}
\usepackage{array}
\usepackage{booktabs}
\usepackage{bm}

\usepackage{mathalfa}
\usepackage{todonotes}

\usepackage{enumitem} % for customizing enumerate labels
\usepackage{hyperref}
\usepackage{geometry}
\usepackage{braket}
\usepackage{physics}

\usepackage{algorithm}
\usepackage{algorithmic}
\usepackage{enumitem}

\usepackage[
backend=bibtex,
style=numeric,
isbn=false,
maxnames=99
]{biblatex}
\DeclareSourcemap{
  \maps[datatype=bibtex, overwrite]{
    \map{
      \step[fieldset=address, null]
      \step[fieldset=editor, null]
      \step[fieldset=location, null]
    }
  }
}

\AtEveryBibitem{%
  	\clearfield{issn} % Remove issn
  	\clearfield{doi} % Remove doi
  	\clearfield{urldate}
	\clearfield{month}
	\clearfield{adress}
  	\ifentrytype{online}{}{% Remove url except for @online
    \clearfield{url}
  }
}
\hypersetup{
    colorlinks=true,
    linkcolor=blue,
    filecolor=magenta,      
    urlcolor=cyan,
    citecolor = blue,
}

\newtheorem{thm}{Theorem}[section]
\newtheorem{cor}[thm]{Corollary}
\newtheorem{assumption}{Assumption}
\renewtheorem{example}[thm]{Example}

\numberwithin{equation}{section}
\newcommand{\rbracket}[1]{\left(#1\right)}      %round
\newcommand{\sbracket}[1]{\left[#1\right]}      %square    
\newcommand{\nnorm}[1]{|\!|\!|#1|\!|\!|}
\newcommand{\innerp}[1]{\langle{#1}\rangle}

\def\mU{\mathcal{U}}
\def\R{\mathbb{R}}
\def\E{\mathbb{E}}

\DeclareMathOperator{\Law}{Law}

\usepackage{todonotes}

\headers{Error Analysis of Generalized Langevin Equations}{Q. Lang and J. Lu}
\fi

\ifarXiv     % arXiv
\documentclass[10pt]{article}  % [12pt] option for the benefit of aging markers

\usepackage{amssymb}
\usepackage{amsthm}    % amssymb package contains more mathematical symbols
\usepackage{graphicx}          % graphicx package enables you to paste in graphics
\usepackage{amsmath}         % best maths package on offer
\usepackage{mathtools}
\usepackage{subfigure}
\usepackage{mathrsfs}
\usepackage{bbm}
\usepackage{caption}
\usepackage{dsfont}
\usepackage{listings}
\usepackage{titlepic}
\usepackage{xcolor}
\usepackage{array}
\usepackage{booktabs}
\usepackage{bm}

\usepackage{mathalfa}
\usepackage{todonotes}

\usepackage{enumitem} % for customizing enumerate labels
\usepackage{hyperref}
\usepackage{geometry}
\usepackage{braket}
\usepackage{physics}

\usepackage{algorithm}
\usepackage{algorithmic}
\usepackage{enumitem}

\usepackage{hyperref}   % load hyperref first
\usepackage{cleveref}   % then cleveref

\usepackage[
backend=bibtex,
style=numeric,
isbn=false,
maxnames=99
]{biblatex}
\DeclareSourcemap{
  \maps[datatype=bibtex, overwrite]{
    \map{
      \step[fieldset=address, null]
      \step[fieldset=editor, null]
      \step[fieldset=location, null]
    }
  }
}

\AtEveryBibitem{%
  	\clearfield{issn} % Remove issn
  	\clearfield{doi} % Remove doi
  	\clearfield{urldate}
	\clearfield{month}
	\clearfield{adress}
  	\ifentrytype{online}{}{% Remove url except for @online
    \clearfield{url}
  }
}
\hypersetup{
    colorlinks=true,
    linkcolor=blue,
    filecolor=magenta,      
    urlcolor=cyan,
    citecolor = blue,
}

\newtheoremstyle{plainupright} % name
  {3pt}                         % Space above
  {3pt}                         % Space below
  {\normalfont}                 % Body font (upright)
  {}                            % Indent amount
  {\bfseries}                   % Theorem head font
  {.}                           % Punctuation after theorem head
  {.5em}                        % Space after theorem head
  {}                            % Theorem head spec

\theoremstyle{plain}  % Use this instead of 'plain'
\newtheorem{thm}{Theorem}[section]
\newtheorem{theorem}{Theorem}[section]
\newtheorem{definition}[thm]{Definition}

\newtheorem{lemma}[thm]{Lemma}
\newtheorem{cor}[thm]{Corollary}

\newtheorem{assumption}{Assumption}
\newtheorem{proposition}[thm]{Proposition}

\theoremstyle{plain}
\newtheorem{remark}[thm]{Remark}
\newtheorem{example}[thm]{Example}

\numberwithin{equation}{section}
\newcommand{\rbracket}[1]{\left(#1\right)}      %round
\newcommand{\sbracket}[1]{\left[#1\right]}      %square    
\newcommand{\nnorm}[1]{|\!|\!|#1|\!|\!|}
\newcommand{\innerp}[1]{\langle{#1}\rangle}

\def\mU{\mathcal{U}}
\def\R{\mathbb{R}}
\def\E{\mathbb{E}}

\DeclareMathOperator{\Law}{Law}

\usepackage{todonotes}

\newcommand{\funding}[1]{{Funding: #1}}
\newcommand{\email}[1]{{#1}}
\fi

\title{Error Analysis of Generalized Langevin Equations with Approximated Memory Kernels \thanks{
\ifjournal  Submitted to the editors DATE. \fi
\funding{The research is supported in part by the National Science Foundation through awards DMS-2309378 and IIS-2403276.}}}
\author{Quanjun Lang\thanks{Department of Mathematics, Duke University, Durham, NC (\email{quanjun.lang@duke.edu}).} \ and Jianfeng Lu\thanks{Departments of Mathematics, Physics, and Chemistry, Duke University, Durham, NC (\email{jianfeng@math.duke.edu}).}}

\date{}

\begin{document}

\maketitle

\begin{abstract}
    We analyze prediction error in stochastic dynamical systems with memory, focusing on generalized Langevin equations (GLEs) formulated as stochastic Volterra equations. We establish that, under a strongly convex potential, trajectory discrepancies decay at a rate determined by the decay of the memory kernel and are quantitatively bounded by the estimation error of the kernel in a weighted norm. Our analysis integrates synchronized noise coupling with a Volterra comparison theorem, encompassing both subexponential and exponential kernel classes. For first-order models, we derive moment and perturbation bounds using resolvent estimates in weighted spaces. For second-order models with confining potentials, we prove contraction and stability under kernel perturbations using a hypocoercive Lyapunov-type distance. This framework accommodates non-translation-invariant kernels and white-noise forcing, explicitly linking improved kernel estimation to enhanced trajectory prediction. Numerical examples validate these theoretical findings.
    %    To our knowledge, these are the first trajectory-level stability bounds for GLEs that scale directly with an explicit weighted kernel discrepancy.
\end{abstract}

\tableofcontents

% \tableofcontents
\section{Introduction}
The generalized Langevin equation (GLE) describes the dynamics of a particle influenced by conservative forces, friction with memory, and stochastic forcing. The GLE for position \( X_t \in \mathbb{R}^d \) and velocity \( V_t\) is given by
\begin{equation*}
    dX_t = V_t dt, \quad dV_t = \left(-\gamma V_t -\nabla U(X_t) - \int_0^t K(t-s)\, V_s \, ds \right)dt+ d\eta(t),
\end{equation*}
where \( U(x):\R^d \to \R \) is a confining potential, $\gamma>0$ is the friction parameter, and \( \eta(t) \in \mathbb{R}^d \) is a random noise process. The kernel \( K(t) \) accounts for history-dependent interaction, and in thermal equilibrium, the fluctuation-dissipation theorem (FDT) requires the noise to satisfy
\begin{equation*}
    \mathbb{E}[\eta(t) \eta(s)^\top] = \beta^{-1} K(t-s),
\end{equation*}
where \( \beta \) is the inverse temperature. In the absence of external forcing, i.e., when \( \nabla U \equiv 0 \), the second-order GLE can be reduced to a first-order stochastic integro-differential equation for the velocity:
\begin{equation} \label{eq:first_order_gle}
    dV_t = \left(-\gamma V_t -\int_0^t K(t-s)\, V_s \, ds\right)dt + d\eta(t).
\end{equation}
Such memory-dependent equations arise naturally in Mori-Zwanzig coarse-graining of high-dimensional systems, where non-Markovian terms are essential to capture effective dynamics \cite{zwanzig2001nonequilibrium, grabert2006projection, li2010coarse}.

%In practice, many applications approximate the noise \( \eta(t) \) as white, even when the system exhibits memory through the kernel \( K(t) \). This leads to simplified, non-Markovian models where the noise is delta-correlated:
%\begin{equation}
%    \mathbb{E}[\eta(t) \eta(s)^\top] = \sigma^2 \delta(t-s) I_d,
%\end{equation}
%decoupling the memory kernel from the noise statistics. 

In practice, it is common to model the forcing $\eta(t)$ as mean-zero Gaussian white noise with fixed covariance $\sigma\sigma^\top \in \mathbb{R}^{d\times d}$, independent of the memory kernel $K(t)$.
This leads to a simplified non-Markovian model with delta-correlated noise:
\[
    \mathbb{E}[\eta(t)] = 0, \qquad
    \mathbb{E}\!\left[\eta(t)\eta(s)^\top\right] = \sigma\sigma^\top\,\delta(t-s).
\]
Moreover, non-translation-invariant memory kernels are also used, that is, \( K(t,s) \) that depend on both time variables independently. This direction is motivated in part by recent developments in large language models (LLMs), particularly the HiPPO state-space model framework \cite{gu2020hippo, gu2021efficiently}, which represents history using projections onto time-evolving orthogonal polynomial bases and yields explicit time-dependent kernels. Such constructions naturally lead to kernels with explicit time dependence and are well-suited for modeling non-stationary and transient dynamics. With this broader view, we focus on two simplified yet widely applicable variants of the GLE framework.
\begin{itemize}
    \item The \emph{first-order GLE with white noise}:
          \begin{equation} \label{eq:first_order_white_noise_intro}
              dV_t = -\gamma V_t \, dt -\int_0^t K(t, s) V_s \, ds\,dt + \sigma dB_t.
          \end{equation}
    \item The \emph{second-order GLE with white noise}:
          \begin{equation} \label{eq:second_order_white_noise_intro}
              dX_t = V_t \, dt, \quad dV_t = -\gamma V_t \, dt - \nabla U(X_t)\, dt - \int_0^t K(t, s)V_s\,ds\,dt + \sigma\,dB_t.
          \end{equation}
\end{itemize}
In both of the above equations, \( B_t\in \R^d\) is a standard Brownian motion, $\sigma \in \R^{d\times d}$ is the diffusion matrix, and \( K(t, s) \in \R^{d\times d}\) is a matrix-valued memory kernel.

\begin{remark}
    We emphasize that the first-order GLE above is not obtained as the overdamped (large-friction) limit of the second-order GLE. In the overdamped regime, one eliminates the inertia variable and derives an effective non-Markovian dynamics for the position process alone, with a memory kernel that differs from the one in our formulation of the first-order GLE. % A rigorous analysis of this overdamped limit in the presence of memory is an interesting direction, but it lies beyond the scope of the present work.
\end{remark}

These models capture a broad class of non-Markovian stochastic systems and support modeling, simulation, and inference in physical, biological, financial, and scientific computing \cite{horenko2007data, karmeshu2011neuronal, wehrli2021scale, ruiz2024benefits}. We focus on \emph{error analysis} when the memory kernel \(K(t,s)\) is learned from data or approximated with perturbation. Our goal is to quantify how kernel approximation errors affect trajectory behavior, thereby informing both prediction and interpretability.

%These models are representative of a broad class of stochastic systems with non-Markovian effects and serve as useful tools for modeling, simulation, and inference in physical, biological, and engineered systems with memory. Our primary interest lies in the \emph{error analysis} that arises when the memory kernel \( K(t,s) \) is estimated from observed trajectory data, or approximated using projection methods.  Understanding how the learned or approximated memory kernel influences the behavior of the resulting stochastic process is crucial for both prediction and physical interpretability.

%Such memory-dependent models arise in coarse-grained representations of high-dimensional dynamical systems, where non-Markovian effects are essential to capturing effective dynamics. 

\subsection{Main contributions}

Our main result shows that the estimation error of the memory kernel directly controls the prediction error of generalized Langevin dynamics. In other words, if the approximate kernel is sufficiently well-behaved (in the sense of satisfying mild integrability and decay assumptions), then the difference between the true and perturbed trajectories remains uniformly bounded, with magnitude proportional to the kernel discrepancy. The primary tool we use is the resolvent of Volterra equations, along with a comparison theorem. We summarize the main results as follows.

\begin{theorem}[Informal main result]
    Let $K$ be the true memory kernel and $\tilde K$ an estimated kernel. Consider the generalized Langevin dynamics \eqref{eq:first_order_white_noise_intro} and \eqref{eq:second_order_white_noise_intro}
    driven by $K$ and its perturbed version driven by $\tilde K$, under synchronized noise. Suppose that both kernels satisfy standard regularity and decay fast enough compared to a given function $h$. Then the coupled dynamics are contractive with respect to an appropriate quadratic distance on trajectories. Moreover, the trajectory error decays at least at the rate prescribed by h and is proportional to the $L^2$ kernel estimation error weighted by $h$. See \Cref{thm:Wt_contraction} and \Cref{thm:error_contraction_second} for precise statements and constants.
\end{theorem}

\subsection{Related works}

\paragraph{Memory kernel identification}
Recovering the memory kernel in the GLE is a central problem for constructing accurate non-Markovian models.
The authors in~\cite{lei2016data} proposed a rational approximation method in the Laplace domain to estimate the kernel.
The authors in~\cite{russo2022machine} developed a machine-learning framework that integrates a multilayer perceptron (MLP) with the GLE to learn memory kernels directly from data, enabling data-driven reduced non-Markovian modeling.
The authors in~\cite{bockius2021model} employed the Prony method to approximate autocorrelation functions and represented the memory kernel via finite-dimensional Markovian embeddings.
Our previous work~\cite{lang2024learning} decoupled the learning process into two stages: a regularized Prony method to estimate autocorrelation functions, followed by a balanced Sobolev-norm-based regression that guarantees theoretical performance bounds.
In particular, the kernel \(L^2\) estimation error is rigorously controlled by the error in the estimated autocorrelation functions.

\paragraph{Error analysis}
Analyses of the Mori-Zwanzig formalism provide a priori error estimates for short memory, $t$-model, and hierarchical finite-memory closures, including convergence conditions and computable upper bounds on the memory integral \cite{zhu2018estimation}. These results bound the errors of \emph{closure approximations} rather than providing stability of a prescribed non-Markovian model in terms of an explicit norm for the kernel discrepancy. In \cite{de2012kernel}, the authors used the resolvent-kernel framework to prove that, for second-kind Volterra equations with non-negative convolution kernels, relative perturbations in the kernel lead to bounded relative errors in the solution (including defect-renewal equations as a special case). This is close to our setting. However, they did not consider the specific decay rate of the trajectory error and different types of memory kernels. In \cite{duong2022accurate}, the authors considered the Prony series approximation of the memory kernel and provided error analysis for a simple case of a one-dimensional harmonic oscillator. However, the kernel is assumed to have exponential decay, and the error analysis for more general cases remains open.

\paragraph{Decay rate analysis for equations with memory}
The classical reference \cite{gripenberg1990volterra} provides a comprehensive treatment of Volterra theory, including resolvent bounds, contractivity results, and their stochastic extensions in weighted $L^p$ spaces. These tools are fundamental for analyzing systems with memory and characterizing their decay properties. Later work considered the decay rate using characteristic equations \cite{kordonis1999behavior}. Building on this foundation, Appleby and collaborators investigated Volterra differential equations under various settings \cite{appleby2002non, appleby2002subexponential, appleby2006exact}. In particular, they developed a rigorous framework for describing trajectory decay rates for a broad class of subexponential functions. This line of work was later extended to stochastic Volterra equations in \cite{reynolds2008decay}, though in a slightly different setting where the diffusion term depends linearly on the state variable, leading to trajectories and noise that both decay to zero. Perturbation effects were also examined in \cite{appleby2017memory}, where the analysis relies on direct comparison arguments rather than an $L^2$-based error framework. Our analysis extends these approaches to the GLE context.

On the other hand, for fluctuation-dissipation balanced GLE, where the noise term is correlated with the memory kernel, the decay behavior has been investigated in detail in \cite{glatt2020generalized, herzog2023gibbsian, duong2024asymptotic}. These studies cover systems with power-law and singular memory kernels and provide a careful analysis of their long-time dynamics. In such settings, the coupling between the correlated noise and the perturbed kernel becomes involved. To focus on the kernel-dependent dynamics, we consider a simplified model with additive Gaussian noise and defer the development of consistent correlated-noise formulations to future work.

%It has close connection with Stochastic Volterra equations, that are a class of non-Markovian stochastic differential equations where the drift or diffusion terms depend on the past trajectory through a memory kernel. A typical example takes the form
%\[
%dX_t = \left( \int_0^t K(t, s) X_s \, ds \right) dt + \sigma\, dB_t,
%\]
%where \( K(t, s) \) is a deterministic kernel and \( B_t \) is standard Brownian motion. These equations generalize classical SDEs by incorporating history-dependent dynamics and arise naturally in models with memory, such as generalized Langevin equations driven by white noise. Under suitable regularity assumptions on the kernel, one can establish existence, uniqueness, and moment bounds for the solution.

\paragraph{Langevin equations}
For the Langevin equation, establishing exponential convergence to equilibrium through direct coupling was posed as an open problem in \cite{villani2008optimal}. This question was later addressed in \cite{bolley2010trend}, where the authors employed a synchronous coupling under relatively restrictive assumptions on the potential. Further progress was made in \cite{eberle2019couplings}, where the authors introduced a special Lyapunov function inspired by \cite{mattingly2002ergodicity} and developed a sticky coupling that combines synchronized and reflection couplings. This approach successfully handled certain nonconvex potentials at the borderline between the overdamped and underdamped regimes. The sharpest contraction rate for strongly convex potentials so far was obtained in \cite{cao2023explicit}, and \cite{schuh2024global} further improved the result to achieve a dimension-free convergence rate by introducing two separate distance scales for coupling.

Another line of research focuses on the evolution of the probability law through the Fokker-Planck equation; see, for instance, \cite{desvillettes2001trend, eckmann2003spectral, herau2004isotropic}. However, in the presence of memory effects, such a convenient Fokker-Planck structure is not available. There are works that focus on deriving hierarchical Fokker–Planck equations and on analyzing delayed stochastic systems and their associated evolution equations; see, for example, \cite{giuggioli2019fokker}. An alternative approach uses Markovian embedding, where the memory kernel is approximated by a Prony series, and the Fokker-Planck equation of the resulting extended system is derived. The marginal distribution of the original variables can then be obtained as in \cite{glatt2020generalized}. Nonetheless, performing rigorous error analysis within this framework remains technically challenging.

\paragraph{Positioning}
To our knowledge, a trajectory-level stability theorem that scales linearly with a weighted $L^2$ kernel discrepancy for GLEs with general, non-translation-invariant kernels has not appeared. Our contribution closes this gap by establishing trajectory-level stability bounds whose constants scale linearly with the kernel estimation error measured in a weighted Schur-type $L^2$ norm.

%\begin{theorem}[Trajectory stability under kernel misfit]
%For generalized Langevin dynamics with a memory kernel $K$, the mean-square trajectory error between the true system and a system driven by an estimated kernel $\tilde K$ is bounded by a constant times a weighted kernel misfit. Under mild integrability/decay conditions, the error inherits the same decay rate as the kernel class. 
%See \Cref{thm:Wt_contraction} and \Cref{thm:error_contraction_second} for precise statements and constants.
%\end{theorem}

%This theorem makes precise the intuitive idea that improving the quality of kernel estimation directly improves trajectory prediction. The framework applies to both first- and second-order GLEs with general kernels, and the proofs combine resolvent bounds for Volterra equations with synchronized coupling of the stochastic forcing.

\paragraph{Notations}
We use the notation \(A(t) \lesssim B(t)\) to indicate that \(A(t) \leq c\,B(t)\) for some constant \(c > 0\) independent of \(t\).
The convolution of two functions is written as \((h * g)(t) = \int_0^t h(t-s) g(s)\, ds\), and the \(n\)-fold convolution is denoted by \(h^{*n}\).
The Laplace transform of a function \(k\) is defined as \(\hat{k}(\mu) = \int_0^\infty e^{-\mu s} k(s)\, ds\), whenever the integral is convergent.  In particular, $\hat{k}(0) = \int_0^{\infty} k(s) \, ds$.

We use \(|\cdot|\) to denote vector norms and \(\|\cdot\|\) to denote matrix operator norms.
The \(L^1\)-norm of a function \(f\) is given by \(\|f\|_1 := \int_0^\infty |f(s)|\, ds\).
For a function \(f\), its weighted supremum norm with respect to a positive weight function \(h\) is denoted by \(\|f\|_h\), as introduced in~\eqref{eq_Mh}.
The Schur-type norm of a matrix-valued kernel function \(K\), weighted by a positive function \(h\), is written as \(\nnorm{K}_h\) and defined in \Cref{def_Schur_norm}.

We use \(\mU(\mu)\) to denote the class of functions with a prescribed decay rate, as defined in \Cref{def_subexp} and \Cref{def_exp}.
Note that a negative value of \(\mu\) corresponds to exponential decay.
Throughout the paper, we write \(\tilde{K}\) to denote an approximated or learned version of \(K\), and use the notation \(\delta K := K - \tilde{K}\) for their difference. We adopt the convention that the capital letter $K$ denotes a matrix-valued kernel, while the lower-case $k$ denotes a scalar-valued kernel.

%We use $A(t) \lesssim B(t)$ if $A(t) \leq cB(t)$ for some constant $c>0$ independent of $t$. We represent convolution as $h*g(t) = \int_0^t h(t-s)g(s)ds$ and $n$-fold convolution as $h^{*n}$. Laplace transform is defined as $\hat{k}(\mu) = \int_0^\infty e^{-\mu s}k(s)ds$ wheneve the integral is convergent. We shall use $|\cdot|$ for vector norms and $\|\cdot \|$ as matrix operator norm. The $L^1$ norm of a function $f$ is defined as $\|f\|_1 := \int_0^\infty f(s)ds$. For a function $f$, we denote its weighted sup norm based on a positive function $h$ as $\|f\|_h$ and define it in \eqref{eq_Mh}. 
%The Schur-type norm of a matrix-valued kernel function $k$ weighted by a positive function $h$ is denoted as $\nnorm{k}_h$ and defined in \Cref{def_Schur_norm}. We use $\mU(\mu)$ to represent functions of a certain decay rate and define it in \Cref{def_subexp} and \Cref{def_exp}. Note that negative $\mu$ represents an exponential decay. Throughout the paper, we shall use $\tilde K$ to represent an approximated or learned version of $K$, and similarly for other terms of interest. For simplicity, we use $\delta K := K - \tilde K$. 

\paragraph{Outline}
\Cref{sec:volterra_prelim} develops the Volterra resolvent framework and the weighted spaces used throughout, including subexponential and exponential-type kernels and the key comparison estimate (\Cref{thm_main}). \Cref{sec:first_order} establishes trajectory and perturbation bounds for the first-order GLE with white noise. \Cref{sec:second_order} extends the analysis to the second-order GLE using a hypocoercive Lyapunov metric and proves the stability of the perturbed dynamics. \Cref{sec:numerics} presents numerical examples for subexponential kernels and both GLE models.

\section{Preliminary on Volterra equations and subexponential kernels}\label{sec:volterra_prelim}

In this section, we establish a comparison theorem that converts an integro-differential inequality into an explicit bound with the same decay rate as the kernel.

We let $\mathcal{U}(0)$ denote a class of subexponentially decaying functions (see \Cref{def_subexp}). Moreover, for any $k\in\mathcal{U}(\mu)$ with $\mu\neq0$, there exists $\varphi\in\mathcal{U}(0)$ such that $k(t)=e^{\mu t}\varphi(t)$. In the case $\mu<0$, the kernel $k$ decays slightly faster than $e^{\mu t}$. Precise definitions are given in Section~\ref{sec:exp_sub_exp}.

\begin{theorem}[Volterra comparison]\label{thm_main}
    Let $a>0$ be a constant, $g\ge 0$ be a bounded and continuous function, and suppose that $y\colon \R_+\to\R_+$ satisfies the integro-differential inequality
    \begin{equation}\label{eq_yakg}
        y'(t) \leq -a y(t) + \int_0^t k(t-s)y(s)\,ds + g(t), \quad y(0) = y_0.
    \end{equation}
    where $k\in \mU(\mu)$ for some $\mu\le 0$. If
    \begin{equation*}
        \mu + a > \hat {k}(\mu),
    \end{equation*}
    then there exists a positive constant $C_0<\infty$ which depends on $ a,k,\mu$ such that for all $t\ge 0$,
    \[
        y(t)\leq C_0( y_0k(t) + (k*g)(t)).
    \]
    In particular, when $g\equiv 0$, $y$ has at least the same decay rate as $k$.
\end{theorem}

This result is crucial for controlling trajectory errors of GLEs with approximated memory terms. It enables decay rate analysis beyond the exponential case, including subexponential kernels where Gr\"{o}nwall-type arguments fail. In this section, we prove the theorem as follows. We first introduce the resolvent for Volterra equations and the function class $\mU(\mu)$. We then establish the theorem for subexponential functions with $\mu = 0$ and treat the equality case of \eqref{eq_yakg}, reducing the problem to an integro-differential equation. The inequality case follows from a standard comparison argument. At last, we apply a change of variables to extend the proof to the exponential case with $\mu<0$.

%This bound is central to controlling trajectory errors of GLEs with approximated memory terms. It supports decay–rate analysis beyond the purely exponential regime, including subexponential kernels for which Gr\"{o}nwall-type arguments are ineffective. In this section we prove the theorem as follows: we first introduce the Volterra resolvent and the function class $\mathcal{U}(\mu)$; we then establish the result for subexponential kernels in the equality case of \eqref{eq_yakg}, reducing to an integro–differential equation; next, a comparison argument handles the inequality case; finally, a change of variables extends the proof to exponential-type kernels.

\begin{remark}\label{rmk_ch_eq}
    % The subexponential case has been studied in \cite{appleby2002subexponential, appleby2002non}, while the general case was addressed in \cite{appleby2006exact} and further applied to stochastic Volterra equations in \cite{reynolds2008decay}. 
    The main focus here is on kernels $k$ with only subexponential decay. For exponential decay, one can analyze the characteristic equation related to \eqref{eq_yakg} (cf., e.g., \cite{kordonis1999behavior}),
    \begin{equation}\label{eq_ch_x}
        \lambda + a = \hat{k}(\lambda).
    \end{equation}
    If it admits a root with negative real parts, then $y$ decays exponentially. In contrast, when $k$ decays only subexponentially, such a root may not exist, leaving the decay rate undetermined. Also see \Cref{lem_exp_case} for more details.
\end{remark}

% In this section, we establish an auxiliary result that will be useful in the subsequent analysis. Specifically, we consider the differential inequality
% \begin{equation}
%     y'(t) \leq -a y(t) + \int_0^t k(t-s)y(s)\,ds + g(t),
% \end{equation}
% where $a>0$ is a constant, \( k \) is a nonnegative memory kernel, and \( g \) is a given source function. Our goal is to derive an explicit bound on the function \( y\) in terms of $k$ and $g$. Such a result follows from the resolvent technique for Volterra equations and a comparison argument. This result is particularly useful to control the trajectory error of GLE introduced by approximations of the memory term.

\subsection{Resolvent of Volterra equations}
Consider the following Volterra equation for $t \geq 0$,
\begin{equation}\label{eq:w_Volterra}
    w(t) = \int_0^t h(t-s)w(s)ds + f(t) = (w*h)(t) + f(t).
\end{equation}
% We use the notation $(w*h)(t) = \int_0^t h(t-s)w(s)ds$ to denote the convolution. 
The $n$-fold convolution $h^{* n}$ is defined as $h^{*1} = h$ and $h^{*(n+1)} = h* h^{* n}$ for $n \geq 1$. The solution to the Volterra equation is closely related to the \textit{resolvent}, which is defined through
\begin{equation}\label{eq_r}
    r(t) = h(t) + (h * r)(t).
\end{equation}

\begin{lemma}\label{lem_resolvent}
    Suppose $r$ is the resolvent of $h$, where $h(t) >  0$ is a continuous function with $\norm{h}_1 < 1$. Then the following Neumann series converges uniformly and gives the resolvent operator,
    \begin{equation}\label{eq_resolvent}
        r(t) = \sum_{n=1}^\infty h^{*n}(t).
    \end{equation}
    Moreover, the solution to the Volterra equation \eqref{eq:w_Volterra} is given explicitly by
    \begin{equation}\label{eq_w_f}
        w(t) = f(t) + (r*f)(t).
    \end{equation}
\end{lemma}
\begin{proof}
    Define the operator $L_h f := h*f$. Then from Young's inequality, we have the $L^1$ operator $\norm{L_h}_1 < 1$ if $\norm{h}_{L^1} < 1.$ Then the Neumann series $(I - L_h)^{-1} = \sum_{n = 1}^{\infty}L_h^n$ is uniformly convergent, therefore the function $r$ in \eqref{eq_resolvent} is well-defined. Direct computation shows that \eqref{eq_r} holds. Let $w$ be given by \eqref{eq_w_f} and convolute it with $h$,  it holds that
    \begin{equation*}
        w*h = (f + r*f)*h = f*h + f*(r-h) = f*r = w-f,
    \end{equation*}
    where the second last equality follows from the identity \eqref{eq_r} and the last from \eqref{eq_w_f}. This suggests that the function defined in \eqref{eq_w_f} is the solution to \eqref{eq:w_Volterra}.
\end{proof}

We now consider the integro-differential equation
\begin{equation}\label{eq:Volterra}
    x'(t) = -a x(t) + \int_0^t k(t-s)x(s) ds + g(t), \quad x(0) = x_0.
\end{equation}
Note that this is the case where equality holds for \eqref{eq_yakg}.
\begin{lemma}\label{lem_diff_resolvent}
    Suppose $x_0 \geq 0$, $k(t)\geq 0$ is a bounded continuous function in $L^1(\R^+)$ and $g(t) \geq 0$ is bounded continuous. Then the solution to \eqref{eq:Volterra} can be expressed as \begin{equation}\label{eq_xz}
        x(t) = z(t) x_0 + (z * g)(t),
    \end{equation}
    where $z(t)$ is called the \textit{differential resolvent} that satisfies
    \begin{equation}\label{eq_Volterra_fundamental}
        z'(t) = -a z(t) + \int_0^t k(t-s)z(s)ds, \quad z(0) = 1.
    \end{equation}
    Moreover, $z$ also admits the expression
    \begin{equation*}
        z = e + e*r,
    \end{equation*}
    with $e:= e^{-at}$ and $r$ is the resolvent of $h$ defined in \eqref{eq_resolvent} with $h = k*e$.
\end{lemma}
\begin{proof}
    Direct computation shows that $x$ defined as in \eqref{eq_xz} is a solution to \eqref{eq:Volterra} if $z$ satisfies               \eqref{eq_Volterra_fundamental}. To find the expression of $z$, we shall transfer \eqref{eq_Volterra_fundamental} to a Volterra equation. Note that $\frac{d}{dt}(e^{at}z(t)) = e^{at}(az(t) + z'(t))$, therefore \eqref{eq_Volterra_fundamental} implies
    \begin{equation*}
        \frac{d}{dt}(e^{at}z(t)) = e^{at}\int_0^t k(t-s) z(s)ds.
    \end{equation*}
    Integrating from 0 to $T$ yields
    \begin{equation*}
        z(T) = \int_0^T \int_0^t k(t-s)e^{-a(T-t)}z(s)\,dsdt + e^{-aT}z(0) = h*z + e,
    \end{equation*}
    where $h = k*e$ and $e(t) = e^{-at}$. Let $r$ be the resolvent of $h$, and the result follows from \Cref{lem_resolvent}.
\end{proof}

The asymptotic behavior of $z$ has been extensively studied in \cite{appleby2002subexponential} and related references therein. A key result shows that if $k$ belongs to a class of \textit{subexponential functions}, then $z$ decays at the same rate as $k$, with
\begin{equation*}
    \lim_{t \to \infty} \frac{z(t)}{k(t)} = \frac{1}{\left(a - \int_0^\infty k(s)\,ds\right)^2}.
\end{equation*}
In the following discussion, we first derive an explicit bound $z(t) \leq C k(t)$ valid for all $t\geq 0$ and then obtain a control of the decay rate of $x(t)$. Our approach builds on definitions and analytical techniques similar to those in \cite{appleby2002subexponential}. Then we provide the same bound on $y$ using the comparison theorem, similar to \cite{beesack1969comparison}. Finally, we generalize the result to kernels with exponential decay.

% The proof proceeds by expressing the solution to \eqref{eq_Volterra_fundamental} in the form  
% \[
% z(t) = e(t) + (e * r)(t),
% \]  
% where \( e(t) = e^{-a t} \), and \( r \) satisfies the integral equation  
% \begin{equation}\label{eq_r_h}
%     r(t) = h(t) + (h * r)(t), \quad \text{where \( h = e * k \)}.
% \end{equation}
% The function \( r \) is known as the \emph{resolvent} of \( h \), and it admits the Neumann series representation  
% \begin{equation}\label{eq_resolvent}
% r(t) = \sum_{n=1}^\infty h^{*n}(t).
% \end{equation}  
% Detailed derivation can be found in Theorem 2 of \cite{appleby2002non}.

\subsection{Subexponential functions and exponential-type functions}\label{sec:exp_sub_exp}
Let us recall a class of subexponential functions introduced in \cite{appleby2002subexponential}.
\begin{definition}\label{def_subexp}
    A \textit{subexponential function} is a positive continuous function $h \in L^1(\mathbb{R^+})$ so that
    \begin{align}
         & \lim_{t \to 0+} \frac{h^{*2}(t)}{h(t)} = 0, \label{eq_U0_1}                                                                   \\
         & \lim_{t \to \infty} \frac{h^{*2}(t)}{h(t)} = 2 \int_0^\infty h(s) \, \mathrm{d}s < \infty, \label{eq_U0_2}                    \\
         & \lim_{t \to \infty} \frac{h(t - s)}{h(t)} = 1, \quad \text{uniformly for $0 \leq s \leq S$, for all $S > 0$}. \label{eq_U0_3}
    \end{align}
    The set of subexponential functions is denoted as $\mU(0)$.
\end{definition}

\begin{remark}
    % Equation \eqref{eq_U0_1} is included because we do not only consider positive subexponential functions but also allow $h(0) = 0$, since it holds for any continuous $h$ with $h(0) > 0$. If $h$ is $C^1$ at 0, $h(0) = 0$ and $h'(0)  > 0$, then equation \eqref{eq_U0_1} is satisfied. 
    Chover \textit{et al.}~\cite{chover1973functions} employ Banach algebra methods to demonstrate the following result. Let
    $h(t) > 0$ be a continuous function in \( L^1(\mathbb{R}^+) \) satisfying
    \[
        \lim_{t \to \infty} \frac{(h * h)(t)}{h(t)} = c.
    \]
    Then it follows that,
    \(
    c = 2 \int_0^\infty h(s)\, ds.
    \)
    Moreover, \eqref{eq_U0_3} implies
    \(
    \lim_{t \to \infty}h(t) e^{\gamma t} = \infty \text{ for every } \gamma >0.
    \)
    Examples of subexponential functions include the familiar $h(t) \sim e^{-t^\alpha}$ for $\alpha \in (0, 1)$ and $h(t) \sim (1+t)^{-\alpha}$ for $\alpha > 1$.
\end{remark}

We now introduce a more general class of functions that allows an exponential decay rate.
\begin{definition}\label{def_exp}
    Let $\mu \in \mathbb{R}$. A function $h(t) > 0$ is in $\mathcal{U}(\mu)$ if it is continuous for all $t \geq 0$ and
    \begin{align}
        % &\hat{h}(\mu) = \int_0^\infty h(t) e^{-\mu t} \, dt < \infty, \label{eq_Umu_1}\\
         & \lim_{t \to \infty} \frac{(h * h)(t)}{h(t)} = 2 \hat{h}(\mu) < \infty, \label{eq_Umu_2}                                                       \\
         & \lim_{t \to \infty} \frac{h(t - s)}{h(t)} = e^{-\mu s}, \quad \text{uniformly for } 0 \leq s \leq S, \text{ for all } S > 0. \label{eq_Umu_3}
    \end{align}
\end{definition}
It can be readily verified that $\mathcal{U}(0)$ coincides with the definition of the class of subexponential functions. Moreover, it holds that (see \cite{chover1973functions})
\begin{equation}\label{eq_sub_to_exp}
    h \in \mU(\mu) \Longleftrightarrow  h(t) = \varphi(t) e^{\mu t} , \text{ with } \varphi(t) \in \mU(0).
\end{equation}
Thus, it suffices to analyze the subexponential case and transfer the results to exponential-type functions.

\begin{remark}\label{rmk_exp_exact}
    Note that for $\mu < 0$, the function class $\mU(\mu)$ corresponds to at least exponential decay: any $h \in \mU(\mu)$ decays slightly faster than $e^{\mu t}$. In fact, the pure exponential $k(t) = e^{\mu t} \notin \mU(\mu)$ for $\mu<0$ since it violates \eqref{eq_Umu_2}, as
    $(k*k)(t) = \int_0^t e^{\mu(t-s)}e^{\mu s} ds = te^{\mu t}.$
    We shall use a more direct method to handle this case. See \Cref{lem_exp_case}.
    % Note that the function class $\mU(\mu)$ with $\mu < 0$ corresponds to at least exponential decay, while functions in $\mU(\mu)$ decay slightly faster than $e^{\mu t}$. In fact, $k(t) = e^{\mu t} \notin \mU(\mu)$ for $\mu<0$ since it violates \eqref{eq_Umu_2}, as
    % $$(k*k)(t) = \int_0^t e^{\mu(t-s)}e^{\mu s} ds = te^{\mu t}.$$
    % We shall use a more direct method to handle this case. See \Cref{lem_exp_case}.
\end{remark}

% For a positive continuous function $h(t)$, we define the function space \( BC_h(\mathbb{R}^+) \) to consist of all functions $f(t) > 0$ such that there exists a bounded continuous function \( \delta K \) \jl{it seems an overkill to write the division here as $\delta K$? why not just say that $f/h$ is bounded continuous?} on \( \mathbb{R}^+ \) satisfying \( \delta K f = h \).

For a positive continuous function \(h(t)\), we define the function space \(BC_h(\mathbb{R}^+)\) to consist of all functions \(f(t) > 0\) such that the ratio \(f/h\) is a bounded continuous function on \(\mathbb{R}^+\). For brevity, we write \( BC_h = BC_h(\mathbb{R}^+) \). This space becomes a Banach space when equipped with the norm
\begin{equation}\label{eq_Mh}
    \norm{f}_h = M_h \sup_{t \geq 0} \abs{\frac{f(t)}{h(t)}},\quad \text{with } M_h = \sup_{t \geq 0}\frac{h^{*2}(t)}{h(t)}.
\end{equation}
If $h \in \mU(0)$, such a constant exists because of \eqref{eq_U0_2} and the continuity of $h$ on $\R^+$.
In particular for $f, g \in BC_h$, we have
\begin{equation}
    M_h \frac{\abs{(f * g)(t)}}{h(t)} \leq \int_0^t \abs{\frac{f(t-s)}{h(t-s)}} \abs{\frac{g(s)}{h(s)}}\frac{h(t-s)h(s)}{h(t)}ds \leq \norm{f}_{h} \norm{g}_h \frac{1}{M_h} \frac{h^{*2}(t)}{h(t)},
\end{equation}
which implies
\begin{equation}\label{eq_f_conv_g_fg}
    \norm{f*g}_h \leq \norm{f}_h \norm{g}_h.
\end{equation}

Here we include the two useful lemmas for subexponential functions from \cite{appleby2002subexponential}. The proof is included to complete the discussion.

\begin{lemma}[Lemma 3.7 in \cite{appleby2002subexponential}]\label{lem_eta_B}
    Let $h$ be subexponential. For any $\eta\in(0, 1)$, there is a constant $B > 0$ independent of $n$, such that for all $n \geq 2$,
    \[
        \sup_{0 \leq t \leq B} \frac{h^{*n}(t)}{h(t)} \leq \eta^{n-1}.
    \]
\end{lemma}
\begin{proof}
    The result follows from the inequality
    \[
        \frac{h^{*(n+1)}(t)}{h(t)} = \int_0^t \frac{h^{*n}(s)}{h(s)} \frac{h(t-s) h(s)}{h(t)} \, \mathrm{d}s
        \leq \sup_{0 < s < t} \frac{h^{*n}(s)}{h(s)} \frac{h^{*2}(t)}{h(t)}.
    \]
    Because of \eqref{eq_U0_1}, we can choose $B$ such that
    \[
        \sup_{0 \leq t \leq B} \frac{h^{*2}(t)}{h(t)} \leq \eta.
    \]
    The result follows from induction.
\end{proof}

The second lemma is a generalization of the Kesten-type bound for subexponential distributions \cite[Theorem 4.11]{foss2011introduction}. See also  \cite[Lemma 3.6]{appleby2002subexponential}. The proof is slightly adapted for our interest.
\begin{lemma}\label{lem:hn}
    Let $h$ be subexponential. For each $0 < \varepsilon < 1$, there is a constant $\kappa > 0$ independent of $n$, such that for all $n \geq 1$,
    \begin{equation*}
        \sup_{t \geq 0} \frac{h^{*n}(t)}{h(t)} \leq \kappa [(1 + \varepsilon)\rho]^{n-1},
    \end{equation*}
    where $\rho = \int_0^\infty h(s)ds.$
\end{lemma}

\begin{proof}
    Let $f = h/ \rho$ so that $f$ is a subexponential function with $\int_0^\infty f(s)ds = 1$. We denote that
    \begin{equation*}
        \alpha_n = \sup_{t \geq 0}\frac{f^{*n}(t)}{f(t)}.
    \end{equation*}
    Note that $\alpha_1 = 1$. We aim to prove for some constants $c$ and $ \varepsilon$ independent of $n$ with $(1+\varepsilon)\rho<1$, such that the following holds
    \begin{equation}\label{eq:alpha_n_induction}
        \alpha_{n+1} \leq c + (1 + \varepsilon) \alpha_n,\quad n \geq 1.
    \end{equation}
    Then by induction,
    \begin{equation*}
        \alpha_n \leq c \sum_{k = 0}^{n-2} (1+\varepsilon)^k  + (1 + \varepsilon)^{n-1} \leq \frac{c}{\varepsilon}[(1 + \varepsilon)^{n-1} - 1] + (1 + \varepsilon)^{n-1} \leq (c/\varepsilon + 1)(1+\varepsilon)^{n-1},
    \end{equation*}
    At last, since $h^{*n}(t) = \rho^n f^{*n}(t)$, it follows
    \begin{equation*}
        \sup_{t \geq 0}\frac{h^{*n}(t)}{h(t)} = \sup_{t \geq 0}\frac{\rho^n f^{*n}(t)}{\rho f(t)} = \kappa [(1 + \varepsilon)\rho]^{n-1},
    \end{equation*}
    where
    \begin{equation*}
        \kappa = c/\varepsilon + 1.
    \end{equation*}

    To show \eqref{eq:alpha_n_induction}, we first notice that for $n \geq 2$,
    \begin{equation}\label{eq:alpha_n+1}
        \alpha_{n+1} \leq
        \sup_{0 \leq  t \leq T} \int_0^t \frac{f^{*n}(s) f(t - s)}{f(t)}\, ds
        + \sup_{t > T} \int_0^A \frac{f^{*n}(s) f(t - s)}{f(t)}\, ds
        + \sup_{t > T} \int_A^t \frac{f^{*n}(s) f(t - s)}{f(t)}\, ds,
    \end{equation}
    with $0 < A < T$ chosen large enough so that
    \begin{align}
        \left| \int_0^t \frac{f(s) f(t - s)}{f(t)} \, ds - \int_0^A f(s) \, ds \right|
                                                                        & < \left(1 + \tfrac{1}{2} \varepsilon \right), \quad & t \geq A,                    \\
        \sup_{0 \leq s \leq A} \left| \frac{f(t - s)}{f(t)} - 1 \right| & \leq \tfrac{1}{2} \varepsilon, \quad                & t \geq T. \label{eq_t_T_eps}
    \end{align}
    The first inequality follows from \eqref{eq_U0_2} together with the normalization condition of $f$, while the second is a direct consequence of \eqref{eq_U0_3}. Then, for the third term in \eqref{eq:alpha_n+1}, since
    \[
        \int_A^t \frac{f(t - s) f(s)}{f(t)} \, ds
        = \int_0^t \frac{f(t - s) f(s)}{f(t)} \, ds
        - \int_0^A f(s) \, ds
        - \int_0^A f(s) \left[ \frac{f(t - s)}{f(t)} - 1 \right] ds,
    \]
    it follows that
    \[
        \sup_{t > T} \int_A^t \frac{f(t - s) f(s)}{f(t)} \, ds \leq (1 + \varepsilon),
    \]
    therefore
    \begin{equation*}
        \sup_{t > T} \int_A^t \frac{f^{*n}(s) f(t - s)}{f(t)}\, ds  = \sup_{t > T} \int_A^t \frac{f^{*n}(s)}{f(s)} \frac{f(s)f(t - s)}{f(t)}\, ds \leq (1+\varepsilon) \alpha_n.
    \end{equation*}
    For the first term in \eqref{eq:alpha_n+1}, since $f$ is a subexponential function, we can choose constants \( 0 < \eta < 1 \) and \( 0 < B < A \) for $f$ by \Cref{lem_eta_B}. Then, for \( 0 \leq t \leq B \),
    \[
        \int_0^t \frac{f^{*n}(s) f(t - s)}{f(t)} \, ds
        = \int_0^t \frac{f^{*n}(s)}{f(s)} \cdot \frac{f(s) f(t - s)}{f(t)} \, ds
        \leq \eta^{n - 1} \frac{f^{*2}(t)}{f(t)} \leq \alpha_2.
    \]
    For \( B \leq t \leq T \),
    \begin{align*}
        \int_0^t \frac{f^{*n}(s) f(t - s)}{f(t)} \, ds
         & = \int_0^B \frac{f^{*n}(s)}{f(s)} \cdot \frac{f(s) f(t - s)}{f(t)} \, ds
        + \int_B^t \frac{f^{*n}(s) f(t - s)}{f(t)} \, ds                                             \\
         & \leq \eta^{n - 1} \frac{f^{*2}(t)}{f(t)}
        + \frac{\max_{0 \leq t < \infty} f(t)}{\min_{B \leq t \leq T} f(t)} \int_B^t f^{*n}(s) \, ds \\
         & \leq \alpha_2 + \frac{\max_{0 \leq t < \infty} f(t)}{\min_{B \leq t \leq T} f(t)},
    \end{align*}
    where the second inequality holds since $\int_0^\infty f^{*n}(t)dt = 1$. For the second term in \eqref{eq:alpha_n+1}, we have similarly that
    \begin{align*}
        \sup_{t > T}\int_0^A \frac{f^{*n}(s) f(t - s)}{f(t)} \, ds
        \leq (1 +\varepsilon/2) \int_0^A f^{*n}(s)  ds \leq (1 +\varepsilon/2),
    \end{align*}
    where the inequality follows from \eqref{eq_t_T_eps}. We have established \eqref{eq:alpha_n_induction} for $c_0 = (1 +\varepsilon/2) + \alpha_2 + \frac{\max_{0 \leq t < \infty} f(t)}{\min_{B \leq t \leq T} f(t)}$, including the trivial case of $n = 1$. Hence, the proof is complete.
\end{proof}

We are now ready to characterize the decay rates of $z$ and $x$ in terms of $k$.
It remains to estimate $\|z\|_k$ and $\|x\|_k$, defined analogously to \eqref{eq_Mh}.

\begin{theorem}\label{thm:x_leq_k}
    Let $k \in \mU(0)$ be a subexponential function and $g(t)\geq 0$ be a bounded continuous function. Suppose that
    \begin{equation*}
        a > \int_0^\infty k(s)ds.
    \end{equation*}
    Then the differential resolvent $z$ defined in \eqref{eq_Volterra_fundamental} satisfies $\norm{z}_k <\infty.$
    Moreover for all $t \geq 0$, the solution to \eqref{eq:Volterra} satisfies
    \begin{equation*}
        x(t) \lesssim x_0 k(t) + (k*g)(t).
    \end{equation*}
\end{theorem}
\begin{proof}
    From Theorem 6.2 in \cite{appleby2002subexponential}, $h = e*k$ is subexponential where $e(t) = e^{-at}$. By \eqref{eq_f_conv_g_fg}, we have
    \begin{equation*}
        \norm{z}_k \leq \norm{e}_k + \norm{e}_k \norm{r}_k
    \end{equation*}
    Moreover,
    \begin{equation*}
        \norm{r}_k = M_k \sup_{t \geq 0}\frac{r(t)}{k(t)} \leq \sup_{t \geq 0}\frac{r(t)}{h(t)} \cdot M_k  \sup_{t \geq 0} \frac{h(t)}{k(t)}  = \frac{1}{M_h}\norm{r}_h \norm{h}_k,
    \end{equation*}
    where $M_k:=\sup_{t \geq 0}\frac{k^{*2}(t)}{k(t)}$.
    Since $\norm{e}_k < \infty$ and $\norm{h}_k \leq \norm{e}_k \norm{k}_k = M_k\norm{e}_k$ are finite, we are left to control $\norm{r}_h$. Through the Fubini theorem,
    \begin{equation*}
        \rho = \int_0^\infty h(s)ds = \frac{1}{a} \int_0^\infty k(s)ds.
    \end{equation*}
    Hence $\rho < 1$ provided $a > \int_0^\infty k(s)ds$. We can choose small $\varepsilon = \frac{1-\rho}{2 \rho}$ so that $(1+\varepsilon)\rho  = \frac{\rho + 1}{2} < 1$. By \Cref{lem:hn}, the Neumann series \eqref{eq_resolvent} is uniformly convergent, so that
    \begin{equation*}
        \frac{1}{M_h}\norm{r}_h
        \leq
        \frac{1}{M_h}\sum_{n = 1}^\infty \norm{h^{*n}}_h
        \leq
        \sum_{n = 1}^\infty \sup_{t \geq 0}\frac{h^{*n}(t)}{h(t)}
        \leq
        \kappa\sum_{n = 1}^\infty [(1+ \varepsilon)\rho]^{n-1}
        =
        \frac{\kappa}{1-(1+\varepsilon)\rho} = \frac{2\kappa}{1-\rho}.
    \end{equation*}
    Therefore
    $$
        \norm{z}_k = \norm{e}_k(1 + \norm{r}_k) \leq \norm{e}_k(1 + {M_k}{\norm{e}_k}\frac{2\kappa}{1-\rho}) < \infty.
    $$
    % which implies $z(t) \leq C k(t)$, where 
    % % \begin{equation}
    % 	$C = M_k \norm{e}_k\left(1 + \norm{h}_k\frac{\kappa}{1-(1+\varepsilon)\rho}\right).$
    % % \end{equation}
    At last, our result follows from \eqref{eq_f_conv_g_fg} and \eqref{eq_xz},
    $$
        x(t) = x_0 z(t) + z*g(t) \lesssim x_0k(t) + k*g(t),
    $$
    where the implicit constant is
    \begin{equation}\label{eq_D}
        \frac{\norm{z}_k}{M_k} \leq  \norm{e}_k(\frac{1}{M_k} + \norm{e}_k\frac{2\kappa}{1-\rho}) := D_{\{a, k\}}.
    \end{equation}
    % In particular, $x(t) \leq C k(t)$ where $$C = M_k \norm{e}_k\left(1 + \frac{\norm{h}_k}{M_h} \frac{\kappa \lambda^2 + 1- \lambda}{1-\lambda}\right) (x_0 + \norm{g}_k).$$
\end{proof}

\subsection{Proof of \Cref{thm_main}}

% In the subsequent discussion, we will consider the inequality
% \begin{equation}\label{eq:y_ineq}
% 	y'(t) \leq -ay(t) + \int_0^t k(t-s)y(s)ds + g(t)
% \end{equation}
% and give a proof of \Cref{thm_main}. 

% \begin{cor}\label{cor_rate_subexponential}
%     Let $k$ be a subexponential function on $\R^+$ satisfying $a > \int_0^\infty k(s)ds$. If $y$ is a continuous function that satisfies \eqref{eq:y_ineq}, then $$y(t) \lesssim k(t) + (k*g)(t),$$ where the hidden constant is the same constant as in \Cref{thm:x_leq_k}.
% \end{cor}
\begin{proof}
    We first consider the case of $k \in \mU(0)$. Let $w(t) = x(t) - y(t)$, where $x$ is the solution to \eqref{eq:Volterra} with $x(0) = y(0)$. Then $w$ satisfies
    \begin{equation*}
        w'(t) \geq -a w(t) + \int_0^t k(t-s)w(s)ds, \quad w(0) = 0.
    \end{equation*}
    Observe that $e^{-at}(e^{at}w(t))' = w'(t) + aw(t)$, so multiplying the inequality by $e^{at}$ and integrate, we obtain
    \begin{equation*}
        w(T) \geq \int_0^T \int_0^t e^{-a(T-t)}k(t-s)w(s)dsdt = (h*w)(T),
    \end{equation*}
    where $h = e*k$. Define \( f := w - h * w \geq 0 \), so that \( w \) satisfies the Volterra equation \eqref{eq:w_Volterra}. By \Cref{lem_resolvent}, the solution admits the representation
    \[
        w = f + r * f,
    \]
    where \( r \) is the resolvent associated with \( h \). Since $k(t) \geq 0$ implies $h(t) \geq 0$, it follows from the Neumann series that $r(t) \geq 0$ for all $t \geq 0$. Given that $f(t) \geq 0$, we conclude that $w(t) \geq 0$, i.e. $y(t) \leq x(t)$. Applying the bound for $x(t)$ from \Cref{thm:x_leq_k}, which requires $a > \int_0^\infty k(s) ds = \hat {k}(0)$, the desired result follows, with the implicit constant $C_0 = D_{\{a, k\}}$ defined as in \eqref{eq_D}.

    For the case of $k \in \mU(\mu)$ for $\mu < 0$, it holds from \eqref{eq_sub_to_exp} that $k(t) = e^{\mu t} \varphi(t)$ where $\varphi(t) \in \mU(0)$. Let $w(t) = e^{-\mu t} y(t) $, then \eqref{eq_yakg} implies
    \begin{equation*}
        e^{\mu t}(w'(t) + \mu w(t)) \leq -a e^{\mu t}w(t) + \int_0^t e^{\mu(t-s)}\varphi (t-s) e^{\mu s} w(s)ds + g(t),
    \end{equation*}
    which simplifies to the previous case when dividing $e^{\mu t}$,
    \begin{equation*}
        w'(t) \leq -(a + \mu) w(t)+ \int_0^t \varphi(t-s) w(s) \,ds + e^{-\mu t}g(t).
    \end{equation*}
    Hence when
    $$a + \mu > \int_0^\infty \varphi(s)\,ds  = \int_0^\infty e^{-\mu s} k(s)ds = \hat {k}(\mu),$$
    it holds that $w(t) \lesssim w(0)\varphi(t) + \varphi*(e^{-\mu t}g(t)), $ which implies
    \begin{equation*}
        y(t) \lesssim y_0e^{\mu t}\varphi(t) + e^{\mu t}\int_0^t \varphi(s) e^{-\mu(t-s)}g(t-s) ds = y_0 k(t) + k*g(t).
    \end{equation*}
    Following the convention in \eqref{eq_D}, the implicit constant is
    $C_0= D_{\{a+\mu,\,e^{-\mu t}k(t)\}}$. Note that it also covers the case of $\mu = 0$.
    %For later use, we extend this to $\mu>0$ by defining
    %    \begin{equation}\label{eq_D_mu}
    %        D_{\{a, k, \mu, y_0\}} = D_{\{a+\mu, e^{-\mu t}k(t), 0, y_0\}}, \quad \text{ when } k \in \mU(\mu), \mu > 0 
    %    \end{equation}
\end{proof}

\begin{example}\label{ex_powerlaw}
    Consider the case that $k(t) = c(t + \alpha)^{-\beta}$. Note that $k\in \mU(0)$ is subexpoential when $\beta > 1$. By \Cref{thm_main}, if
    \begin{equation*}
        a > \int_0^\infty k(t) dt = \frac{c\alpha^{1 - \beta}}{\beta - 1},
    \end{equation*}
    then $y(t) \lesssim k(t) + k*g(t)$. In particular, if $g(t) = 0$, we have $y(t) \lesssim t^{-\beta}$.
\end{example}

%\begin{remark}\label{rmk_exp_exact}
%Note that $k(t) = e^{\mu t} \notin \mU(\mu)$ with $\mu <0$ since it violates \eqref{eq_Umu_2}, as 
%    $$(k*k)(t) = \int_0^t e^{\mu(t-s)}e^{\mu s} ds = te^{\mu t}.$$
%    We shall use a more direct method to handle this case. 
%    % It is clear that for any $\varepsilon > 0$, there exists a $f(t) \in \mU(\mu + \varepsilon)$ such that $k(t) \leq f(t) = e^{(\mu + \varepsilon)t}\varphi (t)$ with $\varphi(t) \in \mU(0)$. Therefore \eqref{eq_yakg} implies 
%    % \begin{equation}
%    %     y'(t) \leq -a y(t) + c \int_0^t f(t-s) y(s)ds + g(s). 
%    % \end{equation}
%    % Suppose $\eqref{eq}$
%    % It follows from \Cref{thm_main} that $y(t) \lesssim(f(t) + (f*g)(t))$ for $f \in \mU(\mu + \varepsilon)$ with any $\varepsilon > 0$, provided 
%    % $$ \mu + \varepsilon + a > c\int_0^\infty e^{(\mu + \varepsilon)s}f(s)ds$$
%    % In this case, we have shown that $x(t)$ has roughly exponential decay with order $\mu$ when $\mu < 0$. Such a result can also be justified using the Gronwall inequality. However, it fails when $\mu = 0$ for which the resolvent technique is necessary. 
%\end{remark}

\subsection{Exactly exponential case via characteristic equations}
We conclude the discussion by incorporating a classical method that uses the Laplace transform and characteristic equations for the case where $k$ exhibits exactly exponential decay.
\begin{lemma}\label{lem_exp_case}
    Suppose \eqref{eq_yakg} holds for any $t \geq 0$, where $a>0$ is a constant and $g(t) \geq 0$ is a bounded continuous function. Assume that the following holds.
    \begin{enumerate}[label=(\roman*), leftmargin=2.5em]
        \item The constant $\gamma^* = \sup\{\Re \lambda: \lambda + a - \hat k(\lambda) = 0\}< \infty$ and fix $\gamma > \gamma^*$.
        \item The function $k \in L^1_{loc}(\R^+)$ so that
    \end{enumerate}
    \begin{equation}\label{eq_kk}
                  \int_0^\infty e^{-\gamma t}\abs{k(t)}dt < \infty, \quad
                  \int_0^\infty t e^{-\gamma t}\abs{k(t)}dt < \infty.
              \end{equation}
    Then we have
    \begin{equation*}
        y(t) \lesssim e^{\gamma t} + \int_0^t e^{\gamma(t-s)}g(s)ds.
    \end{equation*}
\end{lemma}
\begin{proof}
    Because of the comparison used in the proof of \Cref{thm_main}, it suffices to prove for the equality case. From \Cref{lem_diff_resolvent}, we only need to show that the differential resolvent satisfies $|z(t)| \lesssim e^{\gamma t}$ for all $\gamma > \Re \gamma^*$. Take Laplace transform to \eqref{eq_Volterra_fundamental}, we have
    \begin{equation}\label{eq_zhat}
        \hat z(\lambda) = \frac{1}{\lambda + a - \hat {k}(\lambda)}:= \frac{1}{\Delta(\lambda)}.
    \end{equation}
    To determine the decay rate of $z$, we apply the Bromwich integral for the inverse Laplace transform (cf., e.g., \cite{doetsch2012introduction}). Let $m(\xi) := \hat{z}(\gamma + i\xi) = \frac{1}{\Delta(\gamma + i \xi)}$ and apply change of variable $\lambda = \gamma + i\xi$, it holds
    \begin{equation*}
        z(t) = \frac{1}{2\pi i}\lim_{T \to \infty}\int_{\gamma - iT}^{\gamma + iT}e^{\lambda t} \hat{z}(\lambda) d\lambda = \frac{1}{2\pi }e^{\gamma t}R_\gamma(t), \quad R_\gamma(t) := \int_{-\infty}^{\infty} e^{ i \xi t }m(\xi)d\xi.
    \end{equation*}
    Note that the integral is well-defined since $\gamma$ is chosen to avoid all the singularities of $m(\xi)$. We are left to show that $|R_\gamma(t)|$ is bounded for large $t$.

    From the first assumption in \eqref{eq_kk}, $k(t)e^{-\gamma t}\mathbbm{1}_{[0, \infty)}(t)$ is integrable, and its Fourier transform is expressed as $\hat k(\gamma + i \xi)$, which converges to 0 as $\xi \to \infty$ by the Riemann-Lebesgue Lemma. Hence from \eqref{eq_zhat}, it holds $|\Delta(\gamma + i\xi)| = O(\xi)$ as $\xi \to \infty$, which implies $m(\xi) \to 0 $ as $\xi \to \infty$. Apply integration by parts on $R_\gamma(t)$, it holds
    \begin{equation*}
        R_\gamma(t) = -\frac{1}{it}\int_{-\infty}^{\infty} e^{i \xi t} m'(\xi)\,d\xi.
    \end{equation*}
    Note that the second assumption in \eqref{eq_kk} implies that $\hat k$ is differentiable, so is $m$. Hence $m'(\xi) = -\frac{i\Delta '(\gamma + i\xi)}{\Delta(\gamma + i\xi)^2}$. From the second assumption in \eqref{eq_kk}, $tk(t)e^{-\gamma t}$ is integrable and its Fourier transform $-\hat k'(\gamma + i\xi)$ converges to $0$ as $\xi \to \infty$, and hence is bounded. We have shown that $|\Delta(\gamma + i\xi)|^2 = O(\xi^2)$ when $\xi$ is large. To handle the region where $\xi$ is small, we notice from the selection of $\gamma$, there exists $\delta_0:=\inf_{\xi \in \R} |\Delta(\gamma + i\xi)| >0$. Hence, the $R_\gamma(t)$ is bounded for all large $t$, and the proof is complete.
\end{proof}

\begin{example}\label{ex_exp}
    Consider $k(t) = ce^{-\beta t}$ with $\beta > 0$ and $c >0$. Note that $k \notin \mU(-\beta)$ as demonstrated in \Cref{rmk_exp_exact}. We shall use \Cref{lem_exp_case}, which needs to find the (maximal real part of) characteristic roots of \eqref{eq_ch_x}. Since $\hat k(\lambda) = \frac{c }{\beta + \lambda}$, it holds that
    \begin{equation*}
        \lambda + a = \frac{c}{\beta + \lambda} \implies \lambda =  \frac{1}{2}[-(a + \beta) \pm \sqrt{(a + \beta)^2 -4(a \beta - c))}].
    \end{equation*}
    Note that there are always two real roots since the determinant is $(a-\beta)^2 +4c>0$.
    By \Cref{lem_exp_case}, $y(t) \lesssim e^{\gamma t} + (\exp(\gamma \cdot)*g)(t)$ for all $\gamma > \gamma^* = \frac{1}{2}[-(a + \beta) + \sqrt{(a + \beta)^2 -4(a \beta - c))}]$. In particular, if $g(t) = 0$, $y(t) \lesssim e^{(\gamma^* + \varepsilon)t}$ for any $\varepsilon > 0$. It is clear that $\lim_{a \to \infty} \gamma^* = -\beta$.
\end{example}

%%%%%%%%%%%%%%%%%%%%%%%%% first order system  %%%%%%%%%%%%%%%%%%%%%%%%%

\section{Error analysis for first-order equations}\label{sec:first_order}
We consider the first-order Volterra-type stochastic differential equation:
\begin{equation} \label{eq:true}
    dV_t = -\gamma V_t \,dt -\int_0^t K(t, s) V_s \, ds\, dt + \sigma \, dB_t,
\end{equation}
where \( K(t, s) \in \R^{d \times d}\) is the true memory kernel, and \( B_t \) is a standard Brownian motion in \( \mathbb{R}^d \). Now consider a perturbed system with an approximate (or learned) kernel \( \tilde K(t, s) \), we define a new process \( \tilde V_t \) that satisfies
\begin{equation} \label{eq:approx}
    d\tilde V_t = -\gamma \tilde V_t \,dt - \int_0^t \tilde K (t, s) \tilde V _s \, ds\, dt + \sigma \, d \tilde B_t.
\end{equation}
% Let $\delta K = K-\tilde K$. For any kernel function $k(t, s) \in \R^{d \times d}$ and a positive function $h(t)$,  we let 
% \begin{equation}\label{eq:C_kf}
%     \nnorm{K}_h^2 := \sup_{t \geq 0}\int_0^t \frac{\norm{k(t, s)}^2}{h(t-s)}ds,
% \end{equation}
% where $\norm{k(s, t)}$ represents the 2-operator norm of the matrix-valued memory kernel. 
Let $\delta K := K-\tilde K$. To characterize the kernel difference, we define the following norm.
\begin{definition}[Schur-type norm]\label{def_Schur_norm}
    For any kernel function $K(t, s) \in \R^{d \times d}$ and a positive function $h(t)$, we define the following Schur-type weighted norm as
    \begin{equation*}
        \nnorm{K}_h = \sup_{ t \geq 0} \left(\int_0^t \frac{\norm{K(t, s)}^2}{h(t-s)}ds\right)^{1/2},
    \end{equation*}
    where $\norm{K(t, s)}$ inside the integral represents the operator norm for the matrix $K(t, s)$.
\end{definition}

It is straightforward to verify that $\nnorm{\cdot}_h$ is a norm. Moreover, the norm $\nnorm{\cdot}_h$ mirrors the weighted Schur test for integral operators. The Cauchy-Schwarz inequality with the weight $h$ yields
\begin{equation}\label{eq_Schur}
    \left|\int_0^t K(t,s)f(s)\,ds \right|
    \;\le\;
    \Big(\int_0^t \tfrac{\|K(t, s)\|^2}{h(t-s)}\,ds\Big)^{\!\!1/2}
    \Big(\int_0^t h(t-s)\,|f(s)|^2\,ds\Big)^{\!\!1/2}.
\end{equation}
Taking the supremum over $t$ shows that $\nnorm{K}_h$ controls the operator norm of the integral operator with $K$. This resembles the Schur-test (cf.\ \cite{schur1911bemerkungen}; also \cite[Theorem 0.3.1]{sogge2017fourier}), which motivates the term ``Schur-type'' norm.

We make an assumption on the decay rate and the estimation error of the kernels.
\begin{assumption}\label{asmp:decay_f2}
    There exists a function $h \in \mU(\mu)$ (see \Cref{def_exp}) for some $\mu \leq 0$ such that
    \begin{equation*}
        \nnorm{K}_h < \infty, \quad \nnorm{\tilde K}_h < \infty, \quad \nnorm{\delta K}_{h} < \infty.
    \end{equation*}
\end{assumption}

\begin{remark}
    When $K$ is a one-dimensional translation-invariant kernel, the condition $\nnorm{K}_h < \infty$ implies that the decay rate of $K$ is sufficiently fast compared to $\sqrt{h}$. Specifically,
    \begin{equation*}
        \nnorm{K}_h^2 = \sup_{t \geq 0}\int_0^t \frac{\left(K(t-s)\right)^2}{h(t-s)}ds = \int_0^\infty \abs{K(s)}^2 \rho(s)ds = \norm{K}_{L^2(\rho)}^2,
    \end{equation*}
    where the weight function $\rho(s) = \frac{1}{h(s)}$ is supported on $[0, \infty)$. Thus, the condition $\nnorm{\delta K}_h<\infty$ is equivalent to
    \[
        \|K-\tilde K\|_{L^2(\rho)}<\infty,
    \]
    %and thereby characterizes the estimation error of the kernel. A related performance guarantee for learning the memory kernel under this notion is given in \cite{lang2024learning}. \QL{Note that in this case, the weight is exponentially growing, and the kernel estimation error can still be controlled if the weight grows slower than the decay rate of the true kernel, and also the estimated kernel has been truncated. }
\end{remark}

Our goal is to analyze the difference process $V_t - \tilde V_t$. We will first derive the bound for the true trajectory in \eqref{eq:true}, and then control the difference. Both analyses rely on \Cref{thm_main}.

\begin{theorem}\label{thm:Vt_contraction}
    Consider the SDE \eqref{eq:true}. Suppose \Cref{asmp:decay_f2} holds and moreover,
    \begin{equation}\label{eq_mu_2gamma}
        \mu + 2\gamma > 2 \nnorm{K}_h\rbracket{ \hat{h}(\mu)}^{1/2}.
    \end{equation}
    Then there exists positive constants $C_1$ and $c_1$ such that for all $t > 0$,
    \begin{equation*}
        \E|V_t|^2 \leq C_1 h(t) + c_1.
    \end{equation*}
    % where $\varepsilon_1 >\frac{1}{2\gamma}$, $m_1 = 2\gamma - \frac{1}{\varepsilon_1}$ and 
    % $$C_3 = \E|V_0|^2\tilde{C}_{m_1, f} \exp\rbracket{\varepsilon_1 M_K  C_{\norm{K}, f} C_{m_1, f}},$$ $$C_4 = \Tr(\sigma \sigma ^\top) C_{m_1, f}\exp\rbracket{\varepsilon_1 M_K  C_{\norm{K}, f} C_{m_1, f}}.$$
\end{theorem}
\begin{proof}
    From It\^{o}'s formula, Cauchy-Schwarz, and Young's inequality,
    \begin{align*}
        d|V_t|^2
         & = -2\gamma |V_t|^2dt - 2V_t \cdot \int_0^t K(t, s)V_sds + \Tr(\sigma \sigma^\top) dt + 2V_t \cdot \sigma dB_t                                              \\
         & \leq -2\gamma |V_t|^2 + \frac{1}{\varepsilon} |V_t|^2 + \varepsilon \abs{\int_0^t K(t, s) V_s ds}^2 + \Tr(\sigma \sigma^\top) dt + 2V_t \cdot \sigma dB_t,
    \end{align*}
    where $\varepsilon>0$ is a positive constant to be chosen later. By \eqref{eq_Schur}, it holds
    \begin{equation}\label{eq:KV2_leq_MKKV2}
        \abs{\int_0^t K(t, s) V_s ds}^2 \leq \rbracket{\int_0^t \norm{K(t, s)} \abs{V_s}ds}^2 \leq \int_0^t \frac{\norm{K(t, s)}^2}{h(t-s)}ds \,\cdot \int_0^t h(t-s)\abs{V_s}^2 ds \leq \nnorm{K}_h^2 \int_0^t h(t-s)\abs{V_s}^2 ds,
    \end{equation}
    where the last follows from the \Cref{asmp:decay_f2}. Then we take expectation, use Fubini theorem and let $y(t) = \E|V_t|^2$, so that
    \begin{equation*}
        y'(t) \leq -2\gamma y(t) + \frac{1}{\varepsilon}y(t) + \varepsilon \nnorm{K}_h^2 \int_0^t h(t-s) y(s)ds  + \Tr(\sigma \sigma^\top).
    \end{equation*}
    Denoting $a = 2\gamma - \frac{1}{\varepsilon}$, $g(t)$ as a constant $ \Tr(\sigma \sigma ^\top)$ and $k(t - s) = \varepsilon \nnorm{K}_h^2 h(t-s)$, we achieve
    \begin{equation*}
        y'(t) \leq -ay(t) + \int_0^t k(t - s) y(s)ds + g(t).
    \end{equation*}
    In view of \eqref{eq_mu_2gamma}, we choose $\varepsilon = (\nnorm{K}_h(\hat h(\mu))^{1/2})^{-1}$, which ensures the kernel $k$ satisfies the requirement in \Cref{thm_main}. Hence
    \begin{equation*}
        y(t) \leq C_0 \nnorm{K}_h(\hat h(\mu))^{-1/2}( y_0 h(t) + (h*g)(t)),
    \end{equation*}
    for all $t \geq 0$, where $C_0$ is the constant in \Cref{thm_main} with $a$ and $k$ substituted.
    %    where the implicit constant is $D_{\{2\gamma - 1/\varepsilon, \varepsilon \nnorm{K}_h^2 h(t)\}}$. 
    Note that $(h*g)(t) \leq \Tr(\sigma \sigma^\top) \|h\|_1$ since $g$ is a constant, the desired result follows with
    \begin{equation*}
        C_1 = \mathbb{E}|V_0|^2   \nnorm{K}_h(\hat h(\mu))^{-1/2}C_0, \quad c_1 = C_0 \nnorm{K}_h(\hat h(\mu))^{-1/2} \Tr(\sigma \sigma^\top) \|h\|_1.
    \end{equation*}
    It is clear that $c_1 = 0$ if $\sigma = 0$.
\end{proof}

% \begin{remark}
%     In particular, suppose the equation \eqref{eq:true} is deterministic, we have $a = \Tr(\sigma \sigma ^\top) = 0$, so that $|V_t|^2$ decays to 0 at the rate of $h(t)$. 
% \end{remark}

We now prove the error estimate with perturbed memory kernels.

\begin{theorem}\label{thm:Wt_contraction}
    Consider the equations \eqref{eq:true} and \eqref{eq:approx} under a synchronized coupling where $B_t = \tilde B_t$. Suppose \Cref{asmp:decay_f2} and \eqref{eq_mu_2gamma} hold, and moreover,
    \begin{equation}\label{eq_mu_2gamma2}
        \mu + 2\gamma > 2 \nnorm{\tilde K}_h \rbracket{2 \hat{h}(\mu)}^{1/2}.
    \end{equation}
    Then for all $t >0$, it holds
    \begin{equation*}
        \E|V_t - \tilde V_t|^2 \lesssim \left(\E|V_0 - \tilde V_0|^2 + \nnorm{\delta K}_h^2\right)h(t) + \nnorm{\delta K}_h^2\Tr(\sigma\sigma^\top),
    \end{equation*}
    where the constant depends on $\tilde K$, $h$, $\gamma$ and $K$. See \eqref{eq:C2c2} for explicit constants.
\end{theorem}

\begin{proof}
    With synchronized coupling, define $W_t = V_t - \tilde V_t$. It holds that
    \begin{equation*}
        \frac{d}{dt}W_t = -\gamma W_t - \int_0^t \rbracket{K(t, s)V_s - \tilde K(t, s)\tilde V_s} ds.
    \end{equation*}
    From Cauchy-Schwarz and Young's inequality,
    \begin{align}
        \frac{d}{dt}\abs{W_t}^2 & = 2W_t \cdot \frac{d}{dt}W_t = -2\gamma |W_t|^2 - 2 W_t \cdot \int_0^t \rbracket{K(t, s)V_s - \tilde K(t, s)\tilde V_s} ds \notag                                                                   \\
                                & \leq -2\gamma |W_t|^2 + 2\abs{W_t} \cdot \int_0^t \|K(t,s) - \tilde K(t, s)\|\abs{V_s} \, ds + 2\abs{W_t} \cdot \int_0^t \|\tilde K(t, s)\||V_s - \tilde V_s| \, ds \notag                          \\
                                & \leq -2\gamma |W_t|^2 + \frac{2}{\varepsilon}|W_t|^2 + \varepsilon\abs{\int_0^t\norm{\delta K(t, s)}\abs{V_s}ds}^2 + \varepsilon\abs{\int_0^t \|\tilde K(t, s)\||W_s| ds}^2  \label{eq:W_leq_int_int}
    \end{align}
    for some $\varepsilon > 0$ to be determined later.
    % , where at the last step we denote $\delta K(t,s) = K(t,s) - \tilde K(t, s)$. Use H\"{o}lder inequality, 
    % \begin{align}
    %     \abs{\int_0^t\abs{\delta K(t, s)}\abs{V_s}ds}^2 &\leq \int_0^t \abs{\delta K(t, s)}ds \int_0^t \abs{\delta K(t, s)}\abs{V_s}^2ds \leq M_\delta K \int_0^t \abs{\delta K(t, s)}\abs{V_s}^2ds,
    % \end{align}
    % where $M_\delta K = \sup_{t > 0}\int_0^t \abs{\delta K(t, s)} ds$. 
    Apply the estimation in \eqref{eq_Schur}, it holds
    \begin{equation*}
        \E\abs{\int_0^t \|\tilde K(t, s)\|\abs{W_s} ds}^2 \leq \E\int_0^t \frac{\|\tilde K(t, s)\|^2}{h(t-s)}ds \cdot \int_0^t h(t-s) \abs{W_s}^2ds \leq \nnorm{\tilde K}_h^2\int_0^t h(t-s)\E|W_s|^2ds,
    \end{equation*}
    where the last inequality follows from \Cref{asmp:decay_f2} and the Fubini Theorem. Similarly,
    \begin{equation*}
        \E\abs{\int_0^t\norm{\delta K(t, s)}\abs{V_s}ds}^2
        %        \leq 
        %        \nnorm{\delta K}_h^2 \int_0^t h(t-s)\E|V_s|^2ds 
        \leq
        \nnorm{\delta K}_h^2  \int_0^t h(t-s)(C_1 h(s) + c_1)ds \leq \nnorm{\delta K}_h^2(C_1 M_h h(t) + c_1\norm{h}_1),
    \end{equation*}
    where the first inequality follows from \Cref{thm:Vt_contraction} and \Cref{asmp:decay_f2} on $\delta K$ with the constant $C_1$ and $c_1$, and the second follows from the definition of $M_h$ in \eqref{eq_Mh}. At last, \eqref{eq:W_leq_int_int} simplifies to
    \begin{equation*}
        y'(t) \leq -a y(t) + \int_0^t k(t-s)y(s)ds + g(t),
    \end{equation*}
    where $y(t) = \E |W_t|^2$, $a = 2\gamma - \frac{2}{\varepsilon}$, $k(s) = \varepsilon \nnorm{\tilde K}_h^2 h(s)$ and $g(t) =  \varepsilon \nnorm{\delta K}_h^2 (C_1M_h h(t) + c_1\norm{h}_1)$.
    By \Cref{thm_main}, we need to choose $\varepsilon > 0$ such that
    \begin{equation*}
        \mu + 2\gamma - 2/\varepsilon > \varepsilon \nnorm{\tilde K}_h^2 \int_0^{+ \infty} e^{-\mu s}h(s)ds,
    \end{equation*}
    which was enabled by \eqref{eq_mu_2gamma2}. Then it follows
    \begin{equation*}
        y(t) \leq C_0 y(0)k(t) + C_0(k*g)(t),
    \end{equation*}
    where $C_0$ is the constant in \Cref{thm_main} with $k$ and $a$ defined as above.
    %    where the implicit constant is $D_{\{2\gamma - 2/\varepsilon, \varepsilon \nnorm{\tilde K}_h^2 e^{-\mu t}h(t), \E |V_0 - \tilde V_0|^2\}}$. 
    Applying the fact $(h*h)(t) \leq M_h h(t)$ again, it holds
    \begin{equation*}
        k*g(t) \leq \varepsilon \nnorm{\tilde K}_h^2 h*g(t) \leq \varepsilon^2 \nnorm{\tilde K}_h^2 \nnorm{\delta K}_h^2 (C_1M_h^2 h(t) + c_1\norm{h}_1^2).
    \end{equation*}
    Therefore, the desired result follows from
    \begin{equation*}
        \E |V_t - \tilde V_t|^2 \leq C_2 h(t) + c_2,
    \end{equation*}
    where
    \begin{equation}\label{eq:C2c2}
        C_2 = \varepsilon C_0 \nnorm{\tilde K}_h^2 (\E|V_0 - \tilde V_0|^2  + \varepsilon C_1M_h^2 \nnorm{\delta K}_h^2) , \quad c_2 = \varepsilon^2 c_1C_0 \nnorm{\tilde K}_h^2 \nnorm{\delta K}_h^2 \norm{h}_1^2.
    \end{equation}
    Note that $C_1, c_1$ are constants in \Cref{thm:Vt_contraction} and $C_0$ is the constant in \Cref{thm_main}.

    It is clear that the constant $c_2 = 0$ if $\sigma = 0$ (hence $c_1 = 0$) or $\delta K = 0$. Suppose the initial value coincides, then both $C_2$ and $c_2$ scale with the kernel error $\nnorm{\delta K}_h$ quadratically. While the constant can be slightly improved by an optimal choice of $\varepsilon$, this does not change the order of dependence on the kernel error. 
    %    While applying the fact $(h*h)(t) \leq M_h h(t)$ again, we achieved the desired result, $y(t) \lesssim h(t) + c_2$, 
    %    where the implicit constant is denoted as $C_2$ with its exact form omitted for brevity. Note that $c_2 = 0$ if $K = \tilde K$ or $\sigma = 0$.
\end{proof}

\begin{cor}\label{cor_W2_first_order}
    Under the same condition in \Cref{thm:Wt_contraction}, the 2-Wasserstein distance of the Law of $V_t$ and $\tilde V_t$ follows
    \begin{equation*}
        W_2^2(\Law(V_t), \Law(\tilde V_t)) \lesssim \left(\E|V_0 - \tilde V_0|^2 + \nnorm{\delta K}_h^2\right) + \nnorm{\delta K}_h^2\Tr(\sigma\sigma^\top),
    \end{equation*}
    where the implicit constants are given in \eqref{eq:C2c2}.
\end{cor}
\begin{proof}
    By definition,
    \begin{equation*}
        W_2^2(\Law(V_t), \Law(\tilde V_t)) = \inf_{\pi \in \Pi(\Law(V_t), \Law(\tilde V_t))}\mathbb{E}_{(V_t, \tilde V_t) \sim \pi} |V_t -  \tilde V_t|^2 \leq \E_{\pi'}|V_t - \tilde V_t|^2,
    \end{equation*}
    where $\Pi$ represents all couplings between $\Law(V_t)$ and $\Law(\tilde V_t)$, and $\pi'$ represents the synchronized coupling. The result follows from \Cref{thm:Wt_contraction} directly.
\end{proof}

%%%%%%%%%%%%%%%%%%%%%%%%% second order system  %%%%%%%%%%%%%%%%%%%%%%%%%

\section{Error analysis for second-order equations}\label{sec:second_order}
In the presence of an external force, we consider the second-order Volterra-type stochastic differential equation
\begin{equation} \label{eq:second_order_white_noise}
    dX_t = V_t \, dt, \quad dV_t = -\gamma V_t \, dt - u\nabla U(X_t)\, dt - \int_0^t K(t, s)V_s\,ds\,dt + \sigma\,dB_t.
\end{equation}
Here $u, \gamma>0$ are constants. Consider a perturbed system with approximated kernel,
\begin{equation}\label{eq:second_order_white_noise_perturbed}
    d\tilde X_t = \tilde V_t \, dt, \quad d\tilde V_t = -\gamma \tilde V_t \, dt - u\nabla U(\tilde X_t)\, dt - \int_0^t \tilde K(t, s) \tilde V_s \,ds\,dt + \sigma\,d \tilde B_t.
\end{equation}
We will show the contraction of the trajectories $(X_t, V_t)$ and $(\tilde X_t, \tilde V_t)$. We first introduce the assumption on the potential $U$.
\begin{assumption}\label{asmp:U_cvx}
    There exists a positive definite matrix $R \in \R^{d \times d}$ with smallest eigenvalue $\kappa_0 >0$ and a convex function $G:\R^d \to \R$ with $L_G$-Lipschitz continuous gradients, i.e. for all $ x, y \in \R^d,$
    \begin{align}
         & \innerp{\nabla G(x) - \nabla G(y), x-y} \geq 0, \label{eq:G_cvx} \\
         & \abs{\nabla G(x) - \nabla G(y)} \leq L_G\abs{x-y},
    \end{align}
    such that $U(x) = x \cdot Rx/2 + G(x)$ and $\nabla U(x) = Rx + \nabla G(x)$.
\end{assumption}

Note that such an assumption covers the case that $U$ is $\kappa_0$-strongly convex with $L_U$-Lipschitz continuous gradients. In particular, we can set $R = \kappa_0 Id$ so that $G(x) = U(x) - (\kappa_0 /2) |x|^2$, where $Id$ is the $d \times d$ identity matrix. Such a setup follows from \cite{schuh2024global}, which separates the discussion for the parameters.

We will establish a contraction result with the following Lyapunov-type distance function for Langevin dynamics,
$
    r : \mathbb{R}^{2d} \times \mathbb{R}^{2d} \to [0, \infty),
$
where
\begin{equation*}
    r((x,y),(\tilde x,\tilde y)) = \rbracket{\gamma^{-2}u (x-\tilde x)\cdot R(x-\tilde x) + \frac{1}{2} |(1 - 2\lambda)(x - \tilde x) + \gamma^{-1}(y - \tilde y)|^2
        + \frac{1}{2} \gamma^{-2} |y - \tilde y|^2}^{1/2}
\end{equation*}
for $(x,y),(\tilde x,\tilde y) \in \mathbb{R}^{2d}$ with
\begin{equation}\label{eq_lam123}
    \lambda := \min(1/8, \kappa_0 u \gamma^{-2}/2).
\end{equation}
The above definition shall be justified in \Cref{prop_gamma}.  Note that
\begin{equation*}
    r((x, y), (\tilde x, \tilde y))^2 = (x-\tilde x)\cdot A(x-\tilde x) + (x-\tilde x)\cdot B(y-\tilde y) + (y-\tilde y)\cdot C(y-\tilde y).
\end{equation*}
%\[
%\lambda = \min(1/8, \kappa_0 \gamma^{-2}). \tag{2.4}
%\]
%Denote 
%\begin{equation}
%	r^2(t) = Z_t \cdot (A Z_t) + Z_t \cdot (B W_t) + W_t \cdot (C W_t), 
%\end{equation}
where
\begin{equation}\label{eq_ABC}
    A = \gamma^{-2}u R + (1/2)(1-2\lambda)^2 Id, \quad B = (1-2\lambda)\gamma^{-1}Id, \quad   C = \gamma^{-2}Id.
\end{equation}

The strategy is similar to the first-order case. We first derive the contraction for the solution to equation~\eqref{eq:second_order_white_noise}, and then prove the contraction for the discrepancy between \eqref{eq:second_order_white_noise} and \eqref{eq:second_order_white_noise_perturbed}. We will apply \Cref{thm_main} to the Lyapunov function $r$. The following technical result is introduced in \autocite[Theorem 2.1]{schuh2024global} and also in \autocite[Lemma 2.2]{eberle2019couplings}, which is useful for derivative estimations.

\begin{proposition}\label{prop_gamma}
    Let $X_t, V_t, \tilde X_t, \tilde V_t \in \R^d$ and set $Z_t = X_t - \tilde X_t$, $W_t = V_t - \tilde V_t$. Define
    \begin{align}\label{eq_Gamma}
        \Gamma(X_t, \tilde X_t, V_t, \tilde V_t) = (2A Z_t + BW_t )\cdot W_t + (BZ_t +2CW_t)\cdot[-\gamma W_t - u(\nabla U(X_t) - \nabla U(\tilde X_t))],
    \end{align}
    where $A, B, C$ are fixed matrices defined in \eqref{eq_ABC}. Suppose \Cref{asmp:U_cvx} holds and let $\lambda$ be as in \eqref{eq_lam123}. If in addition,
    % \begin{equation}\label{eq:lambda}
    % 	\lambda = \min(1/8, \kappa_0 u \gamma^{-2}/2),
    % \end{equation}
    % and 
    \begin{equation}\label{eq:LG}
        L_G u \gamma^{-2} \leq 3/4,
    \end{equation}
    Then for the above choice of $\lambda$, we have for all $t \geq 0$,
    \begin{equation*}
        \Gamma(X_t, \tilde X_t, V_t, \tilde V_t) \leq -2\lambda \gamma (Z_t \cdot A Z_t + Z_t \cdot B W_t + W_t \cdot  C W_t) = -2\lambda \gamma r((X_t, V_t), (\tilde X_t, \tilde V_t))^2.
    \end{equation*}
\end{proposition}

\begin{proof}
    By \Cref{asmp:U_cvx}, it follows $\nabla U(X_t) - \nabla U(\tilde X_t) = RZ_t + (\nabla G(X_t) - \nabla G(\tilde X_t)). $ Then for the cross terms involving $Z_t$ and $W_t$ in \eqref{eq_Gamma}, it holds
    \begin{equation*}
        2AZ_t \cdot W_t - B Z_t \cdot \gamma W_t - 2 C W_t \cdot u R Z_t = Z_t \cdot (2A - \gamma B - 2u CR) W_t = -2\lambda \gamma Z_t \cdot B W_t.
    \end{equation*}
    Next, by the Cauchy-Schwarz inequality, it holds that
    \begin{equation*}
        2C W_t \cdot \left(u\nabla G(X_t) - u\nabla G(\tilde X_t)\right) \leq \gamma^{-1}\rbracket{|W_t|^2 + \gamma^{-2}u^2\abs{\nabla G(X_t) - \nabla G(\tilde X_t)}^2},
    \end{equation*}
    and together with the $|W_t|^2$ terms in \eqref{eq_Gamma}, it holds
    \begin{equation*}
        BW_t \cdot W_t -2CW_t\cdot \gamma W_t-2C W_t \cdot \left(u\nabla G(X_t) - u\nabla G(\tilde X_t)\right) \leq -2\lambda \gamma W_t\cdot C W_t + \gamma^{-3}u^2\abs{\nabla G(X_t) - \nabla G(\tilde X_t)}^2
    \end{equation*}
    Moreover, for the $Z_t$ terms, it remains to show that
    \begin{equation*}
        -BZ_t \cdot uRZ_t + 2\lambda \gamma  Z_t \cdot A Z_t = Z_t \cdot \sbracket{-u\gamma^{-1}(1-4\lambda)R + \lambda \gamma(1-2\lambda)^2Id}Z_t \leq 0.
    \end{equation*}
    Using \eqref{eq_lam123}, we have
    \begin{equation*}
        -Z_t \cdot u\gamma^{-1}(1-4 \lambda)RZ_t \leq - u\gamma^{-1}(1/2)Z_t \cdot R Z_t \leq -u\gamma^{-1}(1/2)\kappa_0 |Z_t|^2 \leq -\lambda \gamma |Z_t|^2 \leq -\lambda \gamma (1-2\lambda)^2 |Z_t|^2,
    \end{equation*}
    which proves the desired inequality.

    %  we are left to show that 
    % \begin{equation}
    %     -BZ_t \cdot u(\nabla G(X_t) - \nabla G(\tilde X_t)) + \gamma^{-3} u^2 \abs{\nabla G(X_t) - \nabla G(\tilde X_t)}^2 \leq 0.
    % \end{equation}
    We now turn to the remaining terms. From \cite[Theorem 2.1.5]{nesterov2018lectures}, since $G$ is continuously differentiable, convex and has $L_G$-Lipchitz continuous gradients, $G$ is co-coercive, i.e.
    \begin{equation*}
        \abs{\nabla G(x) - \nabla G(\tilde x)}^2 \leq L_G(\nabla G(x) - \nabla G(\tilde x))\cdot (x - \tilde x), \quad \text{for all } x, \tilde x \in \R^d.
    \end{equation*}
    Therefore,
    \begin{multline*}
        -  BZ_t\cdot u \left(\nabla G(X_t) -  \nabla G(\tilde X_t)\right) + \gamma^{-3}u^2 \abs{\nabla G(X_t) - \nabla G(\tilde X_t)}^2
        \\ \leq
        \gamma^{-1}u\sbracket {\gamma^{-2}u L_G-(1-2\lambda)} Z_t \cdot \left(\nabla G(X_t) - \nabla G(\tilde X_t)\right) \leq 0,
    \end{multline*}
    where the last inequality follows from \eqref{eq:G_cvx} and $\lambda \leq 1/8 \leq (1-\gamma^{-2}uL_G)/2$ because of \eqref{eq_lam123} and \eqref{eq:LG}.
\end{proof}

We are ready to introduce the contraction result of the process $(X_t, V_t)$.
\begin{theorem}\label{thm_second_order_XV_contraction}
    Suppose $(X_t, V_t)$ satisfies \eqref{eq:second_order_white_noise}. If \Cref{asmp:decay_f2}, \Cref{asmp:U_cvx} and \eqref{eq:LG} hold, and moreover,
    \begin{equation}\label{eq_eps_XtVt}
        \mu + 2\gamma\lambda > 2(2\nnorm{K}_h^2 \hat{h}(\mu))^{1/2}.
    \end{equation}
    Then let $R_t =  \,r((X_t, V_t), (x^*, 0))^2$, where $x^*$ is the unique minimum of $U(x)$ with $\nabla U(x^*) = 0$, there exist positive constants $C_3$ and $c_3$ such that for all $t>0$,
    \begin{equation*}
        \E R_t \leq C_3 h(t) + c_3.
    \end{equation*}
\end{theorem}
\begin{proof}
    Let $Z_t = X_t - x^*$ and $W_t = V_t$, so that
    \begin{equation*}
        R_t = Z_t \cdot A Z_t + Z_t \cdot B W_t + W_t \cdot  C W_t.
    \end{equation*}
    Therefore, from the It\^{o}'s formula,
    \begin{equation*}
        dR_t = 2 dZ_t \cdot A Z_t + dZ_t \cdot B W_t + Z_t \cdot B d W_t + 2  W_t \cdot C dW_t + \Tr(C \sigma \sigma^\top) dt.
    \end{equation*}
    Note that $dZ_t = dX_t$ and $dW_t = dV_t$, hence from \eqref{eq:second_order_white_noise},
    \begin{align*}
        dR_t     = & \,(2A Z_t + BW_t )\cdot W_t dt + (BZ_t +2CW_t)\cdot \left(-\gamma W_t  - u (\nabla U(X_t) - \nabla U (x^*))- \int_0^t K(t, s) W_s \, ds \right)dt \\
                   & +\Tr(C \sigma \sigma^\top) dt + (BZ_t +2CW_t)\cdot\sigma dB_t.
    \end{align*}
    Separate the Markovian term and denote using $\Gamma$ as in \eqref{eq_Gamma}, we have
    \begin{equation*}
        dR_t = \sbracket{\Gamma(X_t, x^*, V_t, 0)  - \left(B Z_t + 2 C W_t  \right)\cdot \int_0^tK(t, s)W_s ds + \Tr(C \sigma \sigma^\top)}dt + \left(B Z_t + 2 C W_t  \right)\cdot \sigma dB_t.
    \end{equation*}
    By \Cref{prop_gamma}, we have
    \begin{equation*}
        \Gamma(X_t, x^*, V_t, 0) \leq -2 \lambda \gamma R_t.
    \end{equation*}

    %     % Moreover, for the memory term,
    % \begin{align}
    %     \left(B Z_t + 2 C W_t  \right)\cdot \int_0^tK(t, s)W_s ds  
    %     &= ((1-2\lambda)Z_t + 2\gamma^{-1}W_t))\cdot \int_0^tK(t, s) \gamma^{-1}W_sds \\
    %     &\leq \frac{1}{2\varepsilon} \rbracket{|(1-2\lambda)Z_t + \gamma^{-1}{W_t}|^2 + \gamma^{-2} |W_t|^2} + \varepsilon \abs{\int_0^t K(t, s) \gamma^{-1}W_s ds}^2\\
    %     &\leq \frac{1}{\varepsilon} R_t + 2\varepsilon \nnorm{K}_h^2\int_0^t h(t-s) R_s ds,
    % \end{align}

    Furthermore, for the memory term, separate the summation and apply Young's inequality,
    % observe that $B Z_t + 2 C W_t  = \gamma^{-1}((1-2\lambda) Z_t + 2\gamma^{-1} W_t)$ and recall the definition of $R_t$, we have 
    % \begin{equation}\label{eq_BZ2CW}
    %     |(1-2\lambda) Z_t + 2\gamma^{-1} W_t|^2 \leq 2 (|(1-2\lambda)Z_t + \gamma^{-1}W_t|^2 + |\gamma^{-1}W_t|^2) \leq 4 R_t.
    % \end{equation}
    % Hence, from Young's inequality, 
    % \begin{align}
    %     \left(B Z_t + 2 C W_t  \right)\cdot \int_0^tK(t, s)W_s ds  
    %     &\leq 
    %     \frac{1}{2 \varepsilon}|(1-2\lambda)Z_t + 2\gamma^{-1}W_t)|^2 + \frac{\varepsilon}{2}\abs{\int_0^tK(t, s) \gamma^{-1}W_sds}^2\\
    %     & \leq \frac{2}{\varepsilon}R_t + \frac{\varepsilon}{2}
    %     % \\
    %     % &\leq \frac{1}{2\varepsilon} \rbracket{|(1-2\lambda)Z_t + \gamma^{-1}{W_t}|^2 + \gamma^{-2} |W_t|^2} + \varepsilon \abs{\int_0^t K(t, s) \gamma^{-1}W_s ds}^2\\
    %     % &\leq \frac{1}{\varepsilon} R_t + 2\varepsilon \nnorm{K}_h^2\int_0^t h(t-s) R_s ds,
    % \end{align}
    % Moreover, for the memory term,
    \begin{align*}
        \left(B Z_t + 2 C W_t  \right)\cdot \int_0^tK(t, s)W_s ds
         & = \big[((1-2\lambda)Z_t + \gamma^{-1}W_t) + \gamma^{-1}W_t\big]\cdot \int_0^tK(t, s) \gamma^{-1}W_sds                                                        \\
         & \leq \frac{1}{2\varepsilon} \rbracket{|(1-2\lambda)Z_t + \gamma^{-1}{W_t}|^2 + \gamma^{-2} |W_t|^2} + \varepsilon \abs{\int_0^t K(t, s) \gamma^{-1}W_s ds}^2 \\
         & \leq \frac{1}{\varepsilon} R_t + 2\varepsilon \nnorm{K}_h^2\int_0^t h(t-s) R_s ds,
    \end{align*}
    where in the last inequality, the first term follows from the definition of $R_t$, and the second term from the \Cref{asmp:decay_f2}, similar steps in \eqref{eq:KV2_leq_MKKV2} and the fact that $\gamma^{-2}\abs{W_t}^2 \leq 2R_t$.

    Finally, take expectation and let $r(t) = \E R_t, $ it follows that
    \begin{equation*}
        r'(t) \leq -2\lambda \gamma r(t) + \frac{1}{\varepsilon} r(t) + 2 \varepsilon \nnorm{K}_h^2 \int_0^t h(t-s) r(s) ds + \gamma^{-2}\Tr(\sigma \sigma^\top)
    \end{equation*}
    Take $a = 2\lambda\gamma - \frac{1}{\varepsilon}$, $k(t) = 2\varepsilon \nnorm{K}_h^2h(t)$ and $g(t) = \gamma^{-2}\Tr(\sigma \sigma^\top)$. We choose $\varepsilon  = (2\nnorm{K}_h^2 \hat h(\mu))^{-1/2}$, therefore by \eqref{eq_eps_XtVt}, we can apply \Cref{thm_main}, and it follows with appropriately constructed $C_0$ that
    %    \begin{equation}
    %        \mu + 2\gamma\lambda - 1/\varepsilon > 2\varepsilon \nnorm{K}_h^2 \int_0^{+\infty} e^{\mu s} h(s)ds,
    %    \end{equation}
    %    which was enabled by \eqref{eq_eps_XtVt}, then it follows
    \begin{equation*}
        r(t) \leq C_0 r(0)k(t) + C_0(k*g)(t) = C_0 r(0) 2 \varepsilon \nnorm{K}_h^2 h(t) + C_0 2 \varepsilon \nnorm{K}_h^2 \gamma^{-2} \Tr(\sigma \sigma^\top)\norm{h}_1,
    \end{equation*}
    where the last term follows since $g$ is a constant. Taking into account the value of $\varepsilon$, the desired result follows, where
    \begin{equation*}
        C_3 = \sqrt{2}C_0 \E R_0 \nnorm{K}_h(\hat h(\mu))^{-1/2}, \quad c_3 = \sqrt{2}C_0 \nnorm{K}_h(\hat h(\mu))^{-1/2} \gamma^{-2}\Tr(\sigma \sigma^\top) \norm{h}_1.
    \end{equation*}
    It is clear that $c_3 = 0$ if $\sigma = 0$.
    %    Since $g$ is a constant so that $(h*g)(t) < \gamma^{-2}\Tr(\sigma \sigma^\top) \int_0^\infty h(s)ds := c_3$. 
    %    \begin{equation}
    %    	C_2 = 
    %    \end{equation}
    %    
    %    
    %    It is clear that $c_3 = 0$ if $\sigma = 0$. The final result has implicit constant $C_3 = 2\varepsilon \nnorm{K}_h^2D_{\{2\lambda\gamma - 1/\varepsilon, 2\varepsilon \nnorm{K}_h^2h(t), \mu, R_0\}}$.
\end{proof}

We shall now consider the difference between \eqref{eq:second_order_white_noise} and its perturbation \eqref{eq:second_order_white_noise_perturbed}.

\begin{theorem}\label{thm:error_contraction_second}
    Suppose $(X_t, V_t)$ and $(\tilde X_t, \tilde V_t)$ satisfy the equations \eqref{eq:second_order_white_noise}, and \eqref{eq:second_order_white_noise_perturbed}. If \Cref{asmp:decay_f2}, \Cref{asmp:U_cvx}, \eqref{eq:LG} and \eqref{eq_eps_XtVt} hold, and moreover,
    \begin{equation}\label{eq_eps_XtVt_2}
        \mu + 2\gamma \lambda> 4(\nnorm{\tilde K}_h^2 \hat{h}(\mu))^{1/2}.
    \end{equation}
    Then let $R_t =  \,r((X_t, V_t), (\tilde X_t, \tilde V_t))^2$ and assume synchronized coupling $B_t = \tilde B_t$, it holds for all $t > 0$ that
    \begin{equation*}
        \E R_t \lesssim \left(\E R_0 + \nnorm{K - \tilde K}^2_h\right)h(t) + \nnorm{K - \tilde K}^2_h \Tr(\sigma \sigma^\top),
    \end{equation*}
    where the constant depends on $h$, $\tilde K$, $\gamma$, $U$ and $K$. See \eqref{eq:C4c4} for explicit constants.
\end{theorem}

\begin{proof}

    Denote the difference $Z_t = X_t - \tilde X_t$ and $W_t = V_t - \tilde V_t$. Let $B_t = \tilde B_t$, it follows that
    \begin{equation*}
        \left\{
        \begin{aligned}
            dZ_t & = W_t \, dt,                                                                                                                                                                       \\
            dW_t & = -\gamma W_t \, dt - uRZ_t - \left(u\nabla G(X_t) - u\nabla G(\tilde X_t)\right) \, dt - \left( \int_0^t K(t, s) V_s \, ds - \int_0^t \tilde K(t, s) \tilde V_s \, ds \right) dt.
        \end{aligned}
        \right.
    \end{equation*}
    Then we have
    % \begin{align*}
    %     \frac{d}{dt}R_t =& 2 \frac{d}{dt}Z_t \cdot A Z_t + \frac{d}{dt}Z_t \cdot B W_t + Z_t \cdot B \frac{d}{dt} W_t + 2  W_t \cdot C \frac{d}{dt} W_t \\
    %     =& 2 W_t \cdot A Z_t + W_t \cdot B W_t\\
    % 	&+ Z_t \cdot B \sbracket{-\gamma W_t - uRZ_t - u\left(\nabla G(X_t) - \nabla G(\tilde X_t)\right) - \left( \int_0^t K(t, s) V_s \, ds - \int_0^t \tilde K(t, s) \tilde V_s \, ds \right)}\\
    % 	&+2W_t \cdot C \sbracket{-\gamma W_t - uRZ_t - u\left(\nabla G(X_t) - \nabla G(\tilde X_t)\right) - \left( \int_0^t K(t, s) V_s \, ds - \int_0^t \tilde K(t, s) \tilde V_s \, ds \right)}.
    % \end{align*}
    \begin{align*}
        \frac{d}{dt}R_t = & (BZ_t +2CW_t)\cdot \sbracket{-\gamma W_t  - u (\nabla U(X_t) - \nabla U (\tilde X_t))- \rbracket{\int_0^t K(t, s) V_s \, ds - \int_0^t \tilde K(t, s) \tilde V_s \, ds}} \\
                          & +\,(2A Z_t + BW_t )\cdot W_t  +\Tr(C \sigma \sigma^\top).
    \end{align*}
    %    \begin{align*}
    %        dR_t     =& \,(2A Z_t + BW_t )\cdot W_t dt + (BZ_t +2CW_t)\cdot \left(-\gamma W_t  - u (\nabla U(X_t) - \nabla U (x^*))- \int_0^t K(t, s) W_s \, ds \right)dt\\
    % &+\Tr(C \sigma \sigma^\top) dt + (BZ_t +2CW_t)\cdot\sigma dB_t.
    %    \end{align*}
    %    Separate the Markovian term and denote using $\Gamma$ as in \eqref{eq_Gamma}, we have 
    %    \begin{equation}
    %        dR_t = \sbracket{\Gamma(X_t, x^*, V_t, 0)  + \left(B Z_t + 2 C W_t  \right)\cdot \int_0^tK(t, s)W_s ds + \Tr(C \sigma \sigma^\top)}dt + \left(B Z_t + 2 C W_t  \right)\cdot \sigma dB_t.
    %    \end{equation}
    Separate the Markovian term and denote using $\Gamma$ as in \eqref{eq_Gamma}, we have
    \begin{equation*}
        \frac{d}{dt}R_t = \Gamma(X_t, \tilde X_t, V_t, \tilde V_t)  - \left(B Z_t + 2 C W_t  \right)\cdot \rbracket{\int_0^t K(t, s) V_s \, ds - \int_0^t \tilde K(t, s) \tilde V_s \, ds} +\gamma^{-2}\Tr( \sigma \sigma^\top).
    \end{equation*}
    By \Cref{prop_gamma}, it holds that
    \begin{equation*}
        \Gamma(X_t, \tilde X_t, V_t, \tilde V_t) \leq -2\lambda \gamma R_t.
    \end{equation*}

    Moreover, for the memory term,
    \begin{align*}
             &
        \left(B Z_t + 2 C W_t  \right)\cdot \rbracket{\int_0^t K(t, s) V_s \, ds - \int_0^t \tilde K(t, s) \tilde V_s \, ds}                                                                                   \\
        \leq &
        \big[ ((1-2\lambda)Z_t + \gamma^{-1}{W_t}) + \gamma^{-1} W_t\big]\cdot \left(\abs{\int_0^t {\delta K(t, s)}{\gamma^{-1}V_s} \, ds}  + \abs{\int_0^t {\tilde K(t, s)}  {\gamma^{-1} W_s} \, ds} \right) \\
        \leq &
        \frac{1}{\varepsilon}\big[|(1-2\lambda)Z_t + \gamma^{-1}{W_t}|^2 + \gamma^{-2} |W_t|^2\big] + \varepsilon \left(\abs{\int_0^t {\delta K(t, s)}{\gamma^{-1}V_s} \, ds}^2  + \abs{\int_0^t {\tilde K(t, s)}  {\gamma^{-1} W_s} \, ds}^2 \right)
    \end{align*}
    % Moreover, for the memory term,
    % \begin{align}
    %     &
    %     \left(B Z_t + 2 C W_t  \right)\cdot \rbracket{\int_0^t K(t, s) V_s \, ds - \int_0^t \tilde K(t, s) \tilde V_s \, ds} \\
    %     \leq &
    %     \varepsilon\rbracket{|(1-2\lambda)Z_t + \gamma^{-1}{W_t}|^2 + \gamma^{-2} |W_t|^2} + \frac{1}{\varepsilon}\abs{\int_0^t \delta K(t, s) \gamma^{-1}V_s \, ds}^2  + \frac{1}{\varepsilon}\abs{\int_0^t \tilde K(t, s)  \gamma^{-1} W_s \, ds }^2.
    % \end{align}
    % Note that $\abs{B Z_t + 2 C W_t}^2 \leq 2R_t$. 
    Take expectation and let $r(t) = \E R_t$. Recall the similar estimate as in \eqref{eq_Schur}, it holds
    \begin{equation*}
        \E\abs{\int_0^t \tilde K(t, s)  \gamma^{-1}W_s \, ds}^2\leq \nnorm{\tilde K}_h^2 \int_0^t h(t-s) \E | \gamma^{-1}W_s|^2ds\leq 2\nnorm{\tilde K}_h^2 \int_0^t h(t-s) r(s) ds.
    \end{equation*}
    Moreover, \Cref{thm_second_order_XV_contraction} suggests that
    %    where the last inequality follows from , 
    %    since
    \begin{equation*}
        \E |\gamma^{-1}V_t|^2 \leq 2\E r((X_t, V_t), (x^*, 0))^2 \leq 2C_3h(t) + 2c_3.
    \end{equation*}
    Therefore,
    \begin{align*}
        \E \abs{\int_0^t \delta K(t, s)  \gamma^{-1}V_s \, ds}^2 & \leq \nnorm{\delta K}_h^2 \int_0^t h(t-s) \E| \gamma^{-1}V_s|^2 ds
        \leq \nnorm{\delta K}_h^2 \int_0^t h(t-s) (2C_3 h(s) + 2c_3)ds                                                                \\
                                                                 & \leq \nnorm{\delta K}_h^2 (2C_3 M_h h(t) + 2c_3\norm{h}_1)ds.
    \end{align*}
    %    and 
    %    \begin{equation}
    %        \E \abs{\int_0^t \delta K(t, s)  \gamma^{-1}V_s \, ds}^2  \leq \nnorm{\delta K}_h^2 \int_0^t h(t-s) \E| \gamma^{-1}V_s|^2 ds.
    %%        \leq 2C_3C_{\delta K, h}\int_0^t h(t-s)(h(s) + c_3)ds \leq 2C_3C_{\delta K, h}(M_h h(t) + c\norm{h}_{1})
    %    \end{equation}

    %     See \eqref{eq_Phi_V} for a similar step. 

    We finally achieve
    \begin{equation*}
        r'(t) \leq -2\lambda \gamma r(t) + \frac{2}{\varepsilon} r(t) + 2 \varepsilon \nnorm{\tilde K}_h^2 \int_0^t h(t-s) r(s)ds + 2 \varepsilon \nnorm{\delta K}_h^2(C_3M_h h(t) + c_3 \norm{h}_{1}).
    \end{equation*}
    Denote $a = 2\lambda \gamma - 2/\varepsilon$, $k(t) = 2\varepsilon \nnorm{\tilde K}_h^2 h(t)$ and $g(t) = 2 \varepsilon \nnorm{\delta K}_h^2(C_3M_h h(t) + c_3 \norm{h}_{1})$. By \Cref{thm_main}, if we take $\varepsilon$ so that
    \begin{equation*}
        \mu + 2\gamma \lambda - 2/\varepsilon > 2\varepsilon \nnorm{\tilde K}_h^2 \int_0^\infty e^{\mu s} h(s) ds,
    \end{equation*}
    which was enabled by \eqref{eq_eps_XtVt_2}, then with appropraitely constructed $C_0$, it holds
    \begin{equation*}
        r(t) \leq C_0 (r(0)k(t) + k*g(t)).
    \end{equation*}
    We note that
    \begin{equation*}
        k*g(t) \leq 4 \varepsilon^2 \nnorm{\tilde K}_h^2 \nnorm{\delta K}_h^2 \left( C_3 M_h^2 h(t) + c_3 \norm{h}_1^2\right).
    \end{equation*}
    Therefore, the desired result follows from
    \begin{equation*}
        \E R_t \leq C_4 h(t) + c_4,
    \end{equation*}
    where
    \begin{equation}\label{eq:C4c4}
        C_4 = 2\varepsilon C_0 \nnorm{\tilde K}_h^2 \left(\E R_0 + 2\varepsilon C_3 M_h^2 \nnorm{\delta K}_h^2\right), \quad c_4 = 4\varepsilon^2c_3C_0   \nnorm{\tilde K}_h^2 \nnorm{\delta K}_h^2 \|h\|_1^2.
    \end{equation}
    It is clear that $c_4 = 0$ if $\sigma = 0$, as $c_3 = 0$. Moreover, assume identical initial values $X_0 = \tilde X_0$ and $V_0 = \tilde V_0$, both $C_4$ and $c_4$ scale with $\nnorm{\delta K}_h$ quadratically.
\end{proof}

The following Corollary is a natural consequence. The proof is similar to \Cref{cor_W2_first_order}
and hence omitted.
\begin{cor}
    Under the same condition in \Cref{thm:error_contraction_second}, the 2-Wasserstein distance of the Law of $(X_t, V_t)$ and $(\tilde X_t, \tilde V_t)$ follows
    \begin{equation*}
        W_2^2(\Law(X_t, V_t), \Law(\tilde X_t, \tilde V_t)) \lesssim \left(\E r((X_0, V_0), (\tilde X_0, \tilde V_0))^2 + \nnorm{K - \tilde K}^2_h\right)h(t) + \nnorm{K - \tilde K}^2_h \Tr(\sigma \sigma^\top),
    \end{equation*}
    where the implicit constants are defined in \eqref{eq:C4c4}.
\end{cor}

\begin{remark}
    The conditions \eqref{eq_lam123} and \eqref{eq:LG} are satisfied when the friction coefficient \(\gamma\) is large, corresponding to the overdamped regime of the Langevin equation. However, numerical results suggest that contraction may also occur for moderate values of \(\gamma\). This might be a result of the coupling technique applied in the proof, as PDE techniques \cite{cao2023explicit} can establish contraction behavior in the underdamped regime. Extending such quantitative estimates to generalized Langevin equations with memory kernels remains an important and challenging open problem.
\end{remark}

\section{Numerical Examples}\label{sec:numerics}
We first illustrate \Cref{thm_main} numerically. Then we validate the error convergence rate of the first-order (\Cref{thm:Wt_contraction}) and second-order GLE (\Cref{thm:error_contraction_second}).

\subsection{Preliminary results}
\paragraph{Subexponential kernels}

Consider \Cref{ex_powerlaw} in one dimension with
\[
    k(t) = c(t+\alpha)^{-\beta}.
\]
Suppose \(x(t)\) solves \eqref{eq:Volterra} with \(g(t) \equiv 0\) and
\[
    a > \frac{c \alpha^{1-\beta}}{\beta - 1} =: f(\beta).
\]
Then, by \Cref{thm_main}, we have \(x(t) \lesssim (t+\alpha)^{-\beta}\).
In the following experiment, we set \(\alpha = 0.1\), \(c=1\), vary
\(\beta \in [1.05,20]\) and \(a \in [5,50]\), and generate trajectories
for \(t \in [0,100]\). For large \(t \in [5, 100]\), we fit \(x(t) \sim t^{-r}\) and
compare the empirical decay \(r\) with the theoretical exponent \(\beta\).
The results are summarized in \Cref{fig:1}, which confirms that
the empirical decay agrees with the theory whenever \(a\) exceeds the
threshold \(f(\beta)\). For parameter pairs with $a<f(\beta)$, the least-squares fit yields unreliable exponents, so we set the ratio $r / \beta$ to zero in the plot. 

\begin{figure}[t]
    \centering
    \includegraphics[width=0.45\textwidth]{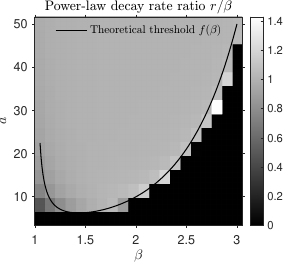}
    \
    \includegraphics[width=0.40\textwidth]{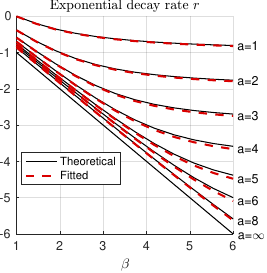}
    \caption{\textbf{Left}: Ratio \(r/\beta\) between fitted and theoretical decay rates
        for trajectories with power-law kernel \(k(t) = (t+0.1)^{-\beta}\).
        Optimal rate (\(r/\beta=1\)) is observed whenever \(a\) exceeds the threshold
        \(f(\beta)\). \textbf{Right}: Comparison of theoretical (black solid) and fitted (red dashed) exponential decay rates with kernel $k(t) = e^{-\beta t}$. Fitted decay rates match the theoretical rates for a large range of $\beta$ and $a$.}
    \label{fig:1}
\end{figure}

\paragraph{Exactly exponential kernels}
Consider \Cref{ex_exp} in one dimension with
\begin{equation*}
    k(t) = ce^{-\beta t}.
\end{equation*}
Suppose $x(t)$ solves \eqref{eq:Volterra} with $g(t) \equiv 0$. Then according to \Cref{thm_main}, the solution has an exponential decay rate of
\begin{equation*}
    p(a, \beta):= \frac{1}{2}[-(a + \beta) + \sqrt{(a + \beta)^2 -4(a \beta - c))}].
\end{equation*}
In the following experiment, we set $c = 1$, vary $\beta \in [1, 6]$, $a \in [1, 10]$ and generate trajectories for $t \in [0, 100]$. For large $t$, we fit $x(t) \sim \exp(-rt)$ and compare the empirical decay rate $r$ with the theoretical rate $p(a, \beta)$. The result is presented in \Cref{fig:1}. The fitted decay rates closely match the theoretical prediction $p(a, \beta)$ across a wide range of $\beta$ and $a$.

\subsection{First-order GLE}\label{sec_5_2}

\begin{figure}[t]
    \centering
    \includegraphics[width=\textwidth]{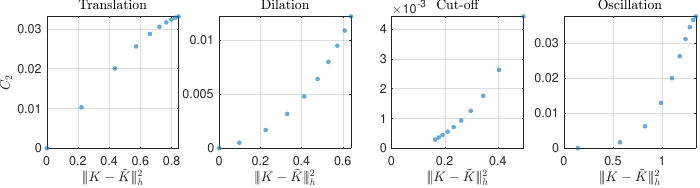}
    \includegraphics[width=\textwidth]{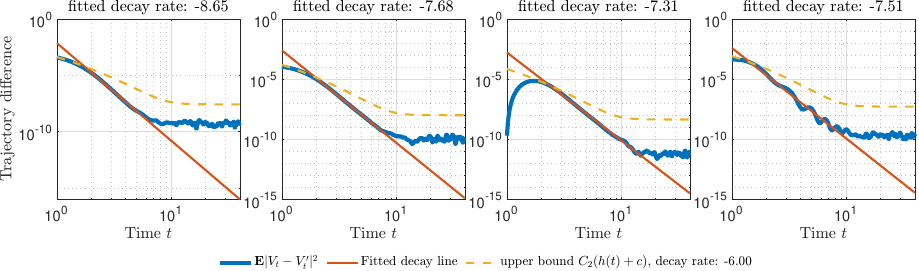}
    \caption{Numerical example for the first-order equation with a power-law kernel \(k(t) = (1+t)^{-4}\) in dimension \(d=1\). \textbf{Top:} Scatter plots showing the relationship between the empirical error upper-bound constant $C_2$ estimated in \eqref{eq:M}, and the squared kernel perturbation error $\nnorm{K- \tilde K}_h^2$ under four types of perturbations: translation, dilation, cut-off, and oscillation. Translation perturbations preserve the kernel decay rate, yielding an approximately linear relation consistent with \Cref{thm:Wt_contraction}. In contrast, dilation and cut-off perturbations modify the kernel decay behavior, leading to a higher-order decrease in $C_2$ with increasing kernel error. Such an effect also appears in oscillation, as it introduces negative memory. \textbf{Bottom:} The decay rate of the trajectory difference, obtained by fitting in the log-log scale for representative trajectories, is consistently close to 8, which is the decay rate of the square of $K(t)$. This suggests that the selected function $h(t) = (1+t)^{-6}$ provides an appropriate and reliable upper bound for the observed decay.}
    \label{fig:kernel_perturbations}
\end{figure}

In this example, we consider the one-dimensional ($d = 1$) first-order GLE \eqref{eq:first_order_gle} with perturbed equation \eqref{eq:approx}.  Suppose the true kernel is translation-invariant and is given by
\begin{equation*}
    K(t) = (t + 1)^{-4}.
\end{equation*}
We assume four kinds of kernel perturbations,
%\begin{align*}
%	&\tilde K_1(t) = (t + 1 + \alpha)^{-4}, &\alpha \in [1, 3];
%	\quad\quad\qquad 
%	&\tilde K_2(t) = (t + 1)^{-4 - \alpha}, &\alpha \in [0, 1];\\
%	&\tilde K_3(t) = (t+1)^{-4} \mathbf{1}_{t \leq \alpha}, &\alpha \in [0.1, 3];
%	\quad\quad\qquad 
%	&\tilde K_4(t) = (t+1)^{-4}\cos(\alpha t), &\alpha \in [0.1, 4],
%\end{align*}
\begin{align*}
    \text{Translation:}\quad & \tilde K_1(t)=(t+1+\alpha)^{-4},                   & \alpha \in [1, 3];   \\
    \text{Cut-off:}\quad     & \tilde K_3(t)=(t+1)^{-4}\,\mathbf 1_{t\le \alpha}, & \alpha \in [0.1, 3]; \\
    \text{Dilation:}\quad    & \tilde K_2(t)=(t+1)^{-4-\alpha},                   & \alpha \in [0, 1];   \\
    \text{Oscillation:}\quad & \tilde K_4(t)=(t+1)^{-4}\cos(\alpha t),            & \alpha \in [0.1, 4].
\end{align*}
Translation corresponds to mis-specified short-term memory magnitude while preserving long-term effects. Dilation reflects an overestimation of the power-law decay rate. The cut-off represents a finite-range approximation of the kernel, and oscillation errors typically arise from frequency-based methods, such as inverse Laplace transforms.

Let $\gamma = 3$ and $\sigma = 10^{-3}$. We choose $h(t) = (1+t)^{-6}$ and generate trajectories for 20 independent batches, each initialized with random but identical initial conditions for both the true and the estimated trajectories. From these simulations, we obtain an empirical estimate of
\begin{equation}\label{eq:M}
    C_2 = \sup_{t > 0}\frac{\E |V_t - \tilde V_t|^2}{h(t) + \Tr(\sigma \sigma^\top)}.
\end{equation}
According to \Cref{thm:Wt_contraction}, $C_2$ is expected to depend linearly on the squared kernel error $\nnorm{K - \tilde K}_h^2$. For various values of $\alpha$ in the kernel perturbation, we present the convergence of the constants $C_2$ with the squared kernel error $\nnorm{K - \tilde K}_h^2$ in the first row of \Cref{fig:kernel_perturbations}.  Although the actual decay rate of the trajectory error may exceed $h$, this choice guarantees that \eqref{eq_mu_2gamma2} holds uniformly across all perturbed kernels, providing a common reference for comparing $C_2$ across experiments. The trajectory decay rate is fitted and presented in the second row of \Cref{fig:kernel_perturbations}.
%
%
%According to \Cref{thm:Wt_contraction}, to make sure $\nnorm{K - \tilde K}_h < \infty$, we take $h(t) = (t+1)^{-6}$ so that \eqref{eq_mu_2gamma2} is satisfied for our estimated kernels. Note that the actual decay rate of the trajectory error can exceed $h$, however, we are taking the minimal rate to encompass all perturbations.

\subsection{Second-order GLE}

\begin{figure}[t]
    \centering
    \includegraphics[width=\textwidth]{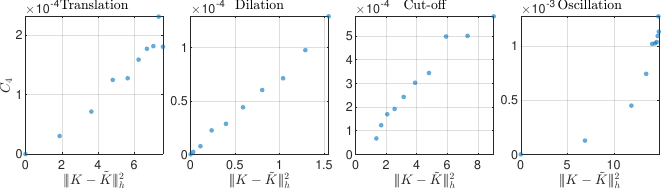}
    \includegraphics[width=\textwidth]{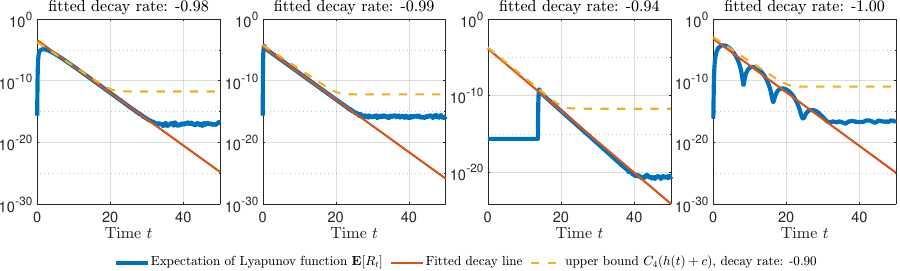}
    \caption{Numerical experiments for three-dimensional $(d = 3)$ second-order dynamics with a matrix-valued kernel exhibiting exponential decay at rate 0.5.
        \textbf{Top:} Scatter plots showing the relationship between the error upper-bound constant $C_4$, defined in \eqref{eq:MM}, and the squared kernel perturbation error $\nnorm{K- \tilde K}_h^2$ under four types of perturbations: translation, dilation, cutoff, and oscillation. The first three scatter plots exhibit an approximately linear relation, consistent with \Cref{thm:error_contraction_second}. The oscillation perturbations deviate from the linear trend, reflecting the effect of negative memory and the comparatively large kernel error in this case. \textbf{Bottom:} Fitted decay rate of the expected value of the Lyapunov distance of the difference  $\mathbb{E}[R_t]$ presented in log scale. The fitted rate is approximately 1, which corresponds to twice the decay rate of the kernel, consistent with \Cref{thm:error_contraction_second}. The upper bound used here, \(h(t) = e^{-0.9t}\), may not be optimal but remains universal across all perturbed kernels.
    }
    \label{fig:2}
\end{figure}

In this example, we study the second-order GLE \eqref{eq:second_order_white_noise} and its perturbed form \eqref{eq:second_order_white_noise_perturbed} in three dimensions $d = 3$. The translation-invariant memory kernel is defined as
\begin{equation*}
    K(t) = e^{-At} = Q e^{-\Sigma t} Q^\top,
\end{equation*}
where \( A \in \mathbb{R}^{3 \times 3} \) denotes a randomly generated symmetric matrix. It admits the eigendecomposition \( A = Q \Sigma Q^{\top} \), with \( \Sigma \) diagonal, and the smallest eigenvalue of \( A \) is fixed to be \( 0.5 \). We examine four types of kernel perturbations as before,
\vspace{-2mm}

\begin{align*}
    \text{Translation:}\quad & \tilde K_1(t) = Q e^{-(\Sigma +\alpha I)t} Q^\top,                  & \alpha \in [0, 0.5]; \\
    \text{Cut-off:}\quad     & \tilde K_3(t) = Q e^{-\Sigma t} \mathbf{1}_{t \leq \alpha} Q^\top , & \alpha \in [1, 20];  \\
    \text{Dilation:}\quad    & \tilde K_2(t) = Q e^{-\Sigma (t+\alpha)} Q^\top,                    & \alpha \in [0, 1];   \\
    \text{Oscillation:}\quad & \tilde K_4(t) = Q e^{-\Sigma t} \cos(\alpha t) Q^\top,              & \alpha \in [0, 1].
\end{align*}

%\begin{align*}
%	&\tilde K_1(t) = U e^{-(\Sigma +\alpha I)t} U^\top, &\alpha \in [0, 0.5];
%	\quad\quad\qquad 
%	&\tilde K_2(t) = U e^{-\Sigma (t+\alpha)} U^\top, &\alpha \in [0, 1];\\
%	&\tilde K_3(t) = U e^{-\Sigma t} \mathbf{1}_{t \leq \alpha} U^\top , &\alpha \in [1, 20];
%	\quad\quad\qquad 
%	&\tilde K_4(t) = U e^{-\Sigma t} \cos(\alpha t) U^\top, &\alpha \in [0, 1].
%\end{align*}
%
We set $\gamma = 10$ and choose $U$ to be a confining potential such that $\kappa = u = 10$ and $L_G = 0.01$. The noise level is set to $\sigma = 10^{-4}$. We choose \( h(t) = e^{-0.9t} \), which decays slightly slower than $e^{-2\cdot 0.5t} = e^{-t}$, so that \Cref{asmp:decay_f2} holds uniformly over all perturbations.
We then generate trajectories over 20 independent batches and compute an empirical estimate of
\begin{equation}\label{eq:MM}
    C_4 = \sup_{t > 0} \frac{\mathbb{E}[R_t]}{h(t) + \operatorname{Tr}(\sigma \sigma^{\top})},
\end{equation}
where \( R_t = r((X_t, V_t), (\tilde X_t, \tilde V_t))^2 \) denotes the squared Lyapunov-type distance between the true and estimated trajectories.
As stated in \Cref{thm:error_contraction_second}, the constant $C_4$ is expected to scale linearly with the squared kernel discrepancy \( \| K - \tilde K \|_h^2 \) since the initial values coincide. The scatter plots are presented in the first row of \Cref{fig:2}. The actual decay rate of the Lyapunov distance function is close to one, exceeding the decay rate of \(h\). The choice of \(h\) is intended to provide a consistent basis for comparing the scaling of \(C_4\) and \(\| K - \tilde K \|_h^2\). The decay rates of several representative trajectories are shown in the second row of \Cref{fig:2}.

\section{Conclusion}
\label{sec:conclusion}

This work establishes a trajectory-wise error analysis framework for generalized Langevin dynamics with approximated memory kernels.
Our analysis quantifies how trajectory discrepancies evolve according to the decay rate of the underlying kernel and provides explicit, time-uniform bounds that scale with the weighted kernel error.
Unlike previous studies that focused on moment or equilibrium estimates, the present approach yields path-wise control over prediction accuracy and extends naturally to subexponentially decaying kernels, for which classical Gr\"{o}nwall-type arguments fail.
The combination of Volterra resolvent theory and synchronized coupling allows for a unified treatment of both first- and second-order dynamics, including non-translation-invariant and matrix-valued kernels. Numerical experiments confirm that the empirical decay of trajectory differences aligns closely with theoretical predictions.

Future directions include incorporating fluctuation-dissipation-consistent noise processes and extending the analysis to the underdamped Langevin regime with small friction~$\gamma$, where hypocoercivity interacts nontrivially with memory effects.
Such extensions would further bridge microscopic stochastic modeling with effective non-Markovian representations in coarse-grained dynamics.

\section*{Acknowledgment}
The research is supported in part by the National Science Foundation through awards DMS-2309378 and IIS-2403276. \ifarXiv
    The authors are grateful to Jing An for many helpful discussions, which have significantly improved the presentation of this paper.
\fi

\printbibliography
\end{document}